%% file: main-twocol.tex
\documentclass[conference]{IEEEtran}

\usepackage[T1]{fontenc}

\usepackage{graphicx}

\newif\iftwocol
\twocoltrue

\usepackage{stfloats}

\newcommand{\titlerunning}[1]{}

\newcommand{\authorrunning}[1]{}

\newcommand{\institute}[1]{}

\newcommand{\keywords}[1]{%
    \begingroup
    \def\and{\(\cdot\)\space}%
    \par\medskip\noindent\textbf{Keywords:} #1%
    \endgroup
}

\newcommand{\inst}[1]{}

\newcommand{\orcidID}[1]{}

\usepackage{amsthm}
\theoremstyle{definition}
\newtheorem{definition}{Definition}

\theoremstyle{plain}

\newtheorem{proposition}{Proposition}

\input{sections/preamble}

\begin{document}

\title{
    Perception with Guarantees:
    Certified Pose Estimation via Reachability Analysis
}

\titlerunning{Perception with Guarantees: Certified Pose Estimation}

\author{
    \IEEEauthorblockN{
        Tobias Ladner\IEEEauthorrefmark{1}\IEEEauthorrefmark{2}\textsuperscript{1},
        Yasser Shoukry\IEEEauthorrefmark{2},
        Matthias Althoff\IEEEauthorrefmark{1}
    }
    \IEEEauthorblockA{
        \IEEEauthorrefmark{1}Technical University of Munich, Germany
    }
    \IEEEauthorblockA{
        \IEEEauthorrefmark{2}University of California, Irvine, USA\\
        Email: tobias.ladner@tum.de
    }
}

\maketitle

\footnotetext[1]{This research was conducted during Tobias Ladner's research stay at the University of California, Irvine.}

\begin{abstract}
\input{sections/abstract}

\end{abstract}

\input{sections/introduction}

\input{sections/background}
\input{sections/content}
\input{sections/evaluation}
\input{sections/related-work}

\input{sections/conclusion}

\newcommand{\ackname}{Acknowledgements}
\newcommand{\discintname}{Disclosure of Interests}
\input{sections/acknowledgements}


\bibliographystyle{splncs04}
\bibliography{bib}


\appendix
\crefalias{chapter}{appendix}
\crefalias{section}{appendix}
\crefalias{subsection}{appendix}

\input{sections/appendix}

\end{document}

%% file: sections/preamble.tex
\usepackage[sets,nn,operations,colors,tikz]{cora-macs}
\CORAexternalizeFiguresByDefault
\iftwocol
  \renewcommand{\includetikzinternalpath}{./figures/externalize-twocol}
  \newcommand{\fittocolwidth}[1]{\resizebox{0.9\columnwidth}{!}{$\displaystyle #1$}}
  \newcommand{\AMP}{}
  \newcommand{\TWOCOLAMP}{&}
\else
  \newcommand{\fittocolwidth}[1]{#1}
  \newcommand{\AMP}{&}
  \newcommand{\TWOCOLAMP}{}
\fi
\usetikzlibrary{positioning}
\newcommand{\inlineLegendBox}[2][20]{%
    \raisebox{0.85mm}{%
        \fcolorbox{#2}{#2!#1!white}{\rule{0pt}{0.25mm}\rule{0.35mm}{0pt}}%
    }%
}%
\newcommand{\inlineLegendCube}[2][20]{%
    \begingroup\tikzexternaldisable
    \tikz[baseline=0pt, x=1ex, y=1ex, line width=0.25pt]{%
        \filldraw[fill=#2!#1!white, draw=#2] (0,0) rectangle (1.2,1.2);%
        \filldraw[fill=#2!#1!white, draw=#2] (0,1.2) -- (0.4,1.6) -- (1.6,1.6) -- (1.2,1.2) -- cycle;%
        \filldraw[fill=#2!#1!white, draw=#2] (1.2,0) -- (1.6,0.4) -- (1.6,1.6) -- (1.2,1.2) -- cycle;%
    }%
    \endgroup
}%
\usepackage{hyperref}

\usepackage{caption}
\usepackage{wrapfig}

\usepackage{xcolor}
\usepackage{lipsum}
\setlipsum{%
    par-before = \begingroup\color{lightgray},
    par-after = \endgroup,
    sentence-before = \begingroup\color{lightgray},
    sentence-after = \endgroup
}

\usepackage{pifont}
\usepackage{cleveref}
\usepackage{amsfonts}
\usepackage{wasysym}

\crefname{section}{Sec.}{Sec.}
\crefname{subsection}{Sec.}{Sec.}
\crefname{figure}{Fig.}{Fig.}
\crefname{algorithm}{Alg.}{Alg.}
\crefname{table}{Tab.}{Tab.}
\crefname{example}{Ex.}{Ex.}
\crefname{equation}{Eq.}{Eq.}
\crefformat{equation}{Eq.~#2#1#3}
\Crefformat{equation}{Eq.~#2#1#3}
\crefrangeformat{equation}{Eqs.~#3#1#4 to~#5#2#6}
\Crefrangeformat{equation}{Eqs.~#3#1#4 to~#5#2#6}
\crefmultiformat{equation}{Eqs.~#2#1#3}{ and~#2#1#3}{, #2#1#3}{, and~#2#1#3}
\Crefmultiformat{equation}{Eqs.~#2#1#3}{ and~#2#1#3}{, #2#1#3}{, and~#2#1#3}
\crefname{definition}{Def.}{Def.}
\crefname{proposition}{Prop.}{Prop.}
\crefname{corollary}{Cor.}{Cor.}
\crefname{theorem}{Thm.}{Thm.}
\crefname{lemma}{Lemma}{Lemmas}
\crefname{appendix}{Appendix}{Appendix}

\renewcommand{\Cref}[1]{\cref{#1}}

\usepackage{booktabs} 
\usepackage{amstext} 
\usepackage{array} 
\newcolumntype{L}{>{$}l<{$}} 
\newcolumntype{R}{>{$}r<{$}} 
\newcolumntype{C}{>{$}c<{$}} 

\usepackage{xfrac}

\usepackage{algorithm}
\usepackage[noend]{algpseudocode}
\usepackage[commentColor=black, italicComments=false]{algpseudocodex}
\AtBeginEnvironment{algorithm}{\tikzexternaldisable}
\AtEndEnvironment{algorithm}{\tikzexternalenable}

\usepackage{bbold}
\newcommand{\indicator}[1]{\mathbb{1}_{\{#1\}}}

\usepackage{thm-restate}

\usepackage{orcidlink}

\usepackage{ tipa }

\newcommand{\cmatrix}[1]{\left[\begin{matrix}
#1
\end{matrix}\right]}
\newcommand{\transpose}[0]{\top}

\usepackage{amssymb}
\newcommand{\opQuadMapMat}[2]{#1\boxdot#2}

\newcommand{\varPG}[0]{P}
\providecommand{\PG}{\contSet{\varPG}}
\newcommand{\shortPG}[2][]{\shortContSet[#1]{#2}{\varPG}}
\newcommand{\shortPH}[3][]{\shortContSet[#1]{#2,#3}{H}}
\newcommand{\numPGvertices}[0]{\nu}
\newcommand{\numPGs}[0]{\rho}
\newcommand{\target}[0]{\contSet{T}}
\newcommand{\shortTarget}[2][]{\shortContSet[#1]{#2}{T}}
\newcommand{\focalLen}[0]{f}
\newcommand{\camera}[0]{C_{\focalLen,w,h}}
\newcommand{\pose}[0]{\xi}
\newcommand{\poseUncertain}[0]{\mathcal{E}}
\newcommand{\poseSpace}[0]{\Xi}
\newcommand{\angleX}[0]{\theta_x}
\newcommand{\angleY}[0]{\theta_y}
\newcommand{\angleZ}[0]{\theta_z}
\newcommand{\image}[0]{I}
\newcommand{\imageOuterApprox}[0]{\widehat{\image}}
\newcommand{\imageNoisy}[0]{\widetilde{\image}}
\newcommand{\imageNoiseLevel}[0]{\tau}
\providecommand{\B}{\mathbb{B}}
\newcommand{\opConvHull}[2][]{\operatortt[#1]{convHull}{#2}}
\providecommand{\V}{\contSet{V}}
\providecommand{\Q}{\contSet{Q}}
\newcommand{\discError}[0]{\contSet{D}}
\newcommand{\numSFs}[0]{\gamma}
\newcommand{\opSupportFunc}[3][]{\operatortt[#1]{supportFunc}{#2,\,#3}}
\newcommand{\refPoint}[0]{B}

\newcommand{\numPoseCandidates}[0]{\zeta}
\newcommand{\cZ}[3]{#1|_{#2,#3}}
\newcommand{\hypercube}[0]{\contSet{A}}

\newcommand{\errorRate}[0]{\delta}

%% file: sections/abstract.tex
Agents in cyber-physical systems are increasingly entrusted with safety-critical tasks.
Ensuring the safety of these agents often requires localizing their pose for subsequent actions.
Pose estimates can, e.g., be obtained from various combinations of lidar sensors, cameras, and external services such as GPS.
Crucially, in safety-critical domains, a rough estimate is insufficient to formally determine safety, i.e., to guarantee safety even in extreme scenarios, and external services may additionally be untrustworthy.
We address this problem by presenting an approach for certified pose estimation in 3D solely from a camera image and a well-known target geometry.
This is realized by formally bounding the pose, which is computed by leveraging recent results from reachability analysis and formal neural network verification.
Our experiments demonstrate that our approach efficiently and accurately localizes agents in both synthetic and real-world experiments.

%% file: sections/introduction.tex
\section{Introduction} \label{sec:introduction}

Agents act increasingly autonomously in real-world scenarios, including safety-critical domains such as autonomous driving~\cite{grigorescu2020survey}.
If one demands safety guarantees in such systems, the systems have to be formally verified~\cite{althoff2010reachability}.
This necessitates a correct perception of the pose of the agent with respect to surrounding obstacles so that the agent can move safely within the environment by selecting only actions that adhere to the given safety specifications~\cite{garcia2015comprehensive},
and might iteratively refine its understanding as it moves through the environment~\cite{clarke2000counterexample}.

Unfortunately, existing localization techniques~\cite{cadena2017past,placed2023survey} -- e.g., based on cameras and lidar sensors -- only provide a rough estimate of the current pose,
which is insufficient for downstream tasks if formal safety is required.
In addition, localization is often performed using deep learning~\cite{chen2023deep}, which is susceptible to adversarial attacks~\cite{goodfellow2015explaining}.
This can lead to wrong predictions for noisy images and under poor weather conditions;
however, such conditions need to be considered in real-world environments.
Moreover, external services, such as GPS, can be interfered with -- e.g., in geopolitically tense scenarios~\cite{pirayesh2022jamming,wu2024gps} -- making them untrustworthy.

We address this problem by presenting an approach to obtain certified pose estimates solely from images taken by an event-based camera~\cite{gallego2020event} and a well-known target geometry.
Such a target geometry can be, e.g., the geometry of a stop sign or standardized runway markings on the ground (\cref{fig:motivation-planar-target}).
Knowing the geometry, we aim to robustly detect the target in an image and return a certified pose estimation of the camera in relation to this target.
We show that such a pure vision-based setting is already sufficient to obtain certified pose estimates.
Our approach can also handle noisy images, typically obtaining a more conservative certified pose estimate depending on the amount of noise,
making it also applicable in real-world settings.

\begin{figure*}
    \centering
    \begin{minipage}[c]{\iftwocol 0.8\linewidth \else \linewidth \fi}
        \centering
        \small (a) \hspace{0.45\linewidth} (b) \\
        \begin{minipage}[c]{0.46\linewidth}
          \centering%
          \includegraphics[width=0.9\linewidth, alt={Runway markings digits 1-10 as by ICAO standard}]{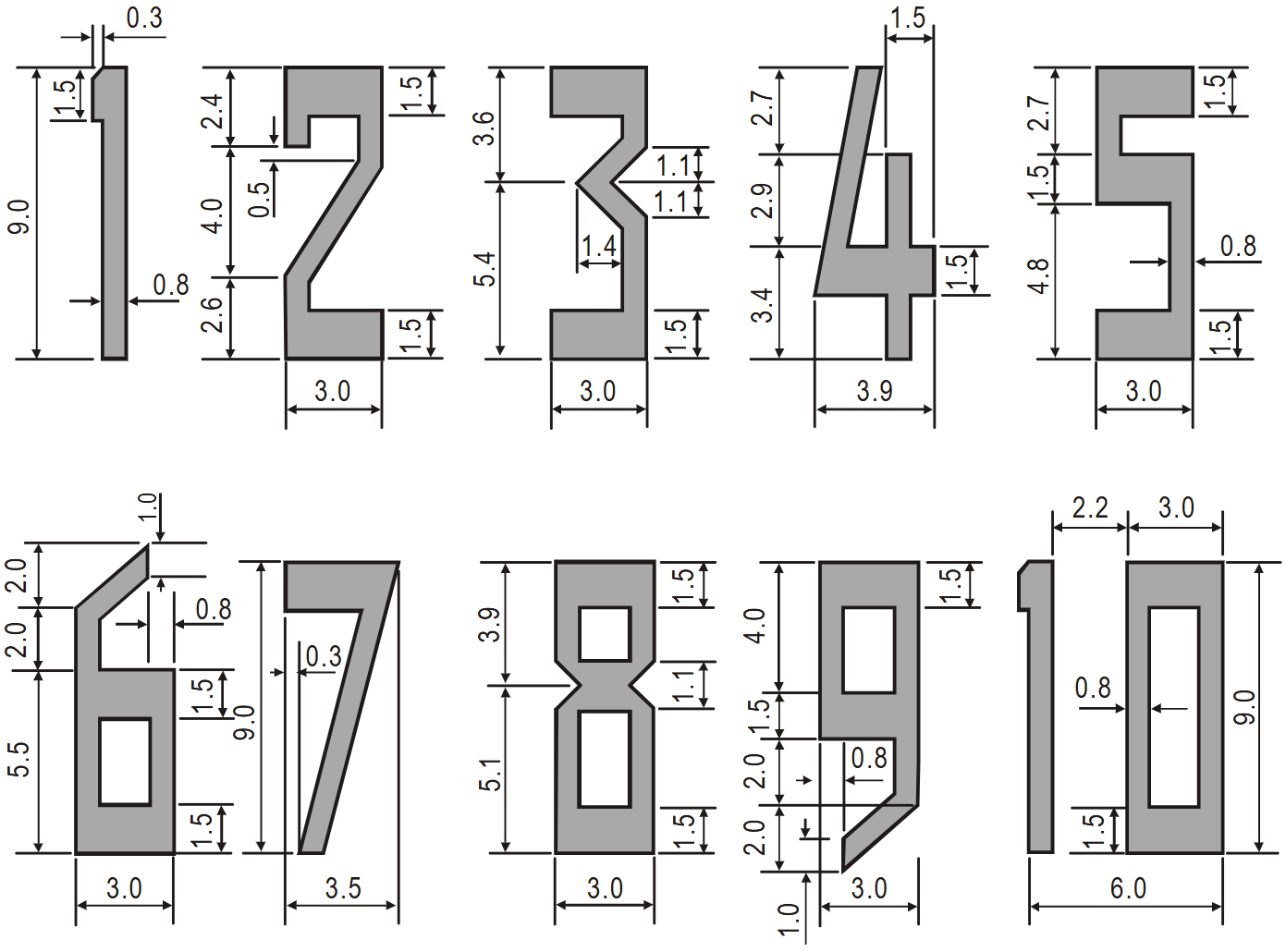}%
        \end{minipage}%
        \begin{minipage}[c]{0.07125\linewidth}
          \centering%
          \includegraphics[width=0.9\linewidth, alt={Runway markings runway as by ICAO standard}]{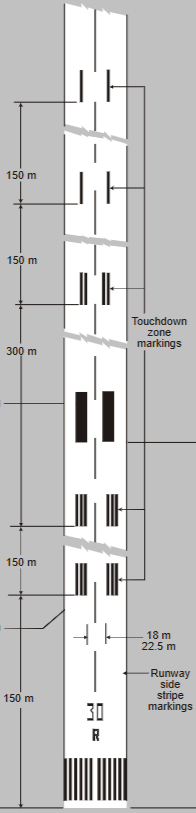}%
        \end{minipage}
        \begin{minipage}[c]{0.46\linewidth}
          \centering%
          \includegraphics[width=0.9\linewidth, alt={Picture of runway with markings visible}]{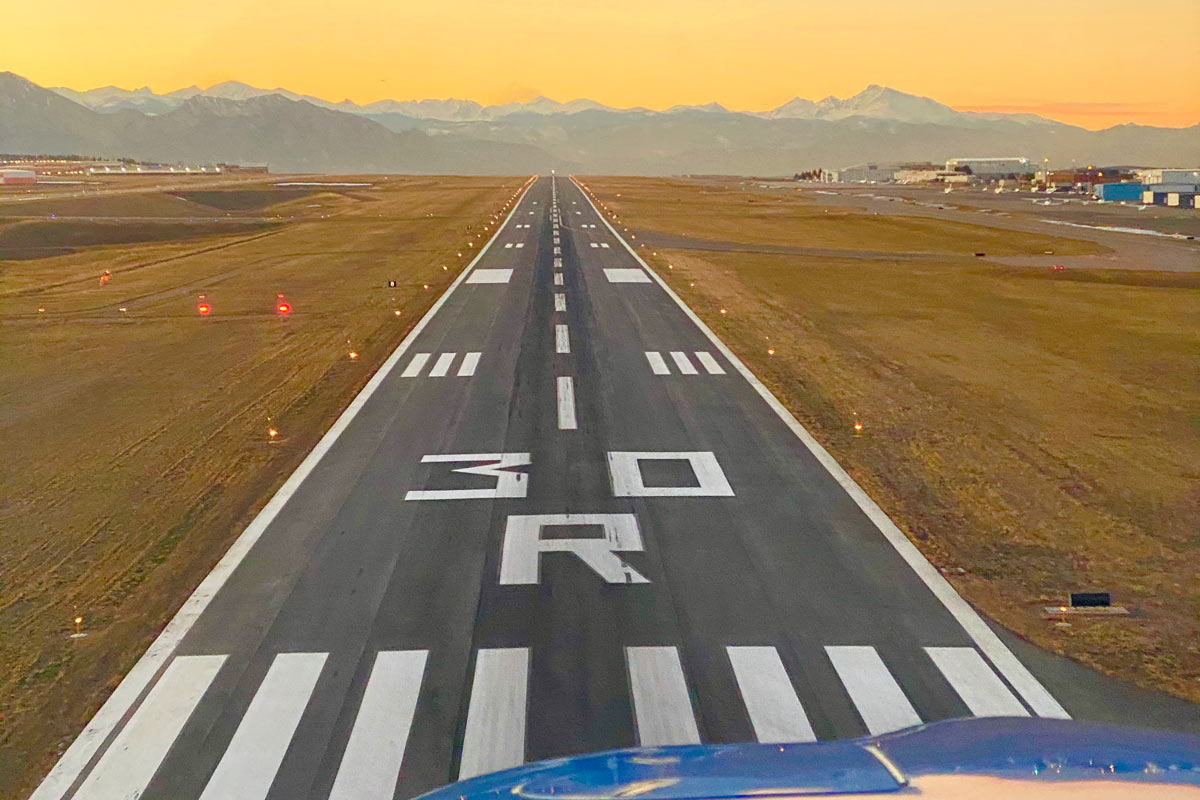}%
        \end{minipage}
    \end{minipage}
    \caption{
    (a) Runway markings as by the ICAO standard~\cite[Chp.~5]{icao2018Aerodrome}.
    (b) Runway markings observed from a landing plane~\cite{runwaywebsite}.}
    \label{fig:motivation-planar-target}
\end{figure*}

In summary, our work marks an important step toward the deployment of autonomous agents in safety-critical domains, and our main contributions are:
\begin{enumerate}
    \item We enclose possible images of a target from an uncertain pose using reachability analysis.
    \item We show how a certified pose estimate given an image can be retrieved, thereby guaranteeing a correct perception.
    \item Our approach naturally incorporates uncertainty in poses, camera parameters, and target geometry.
    \item We obtain tight certified pose estimates on synthetic experiments as well as real-world experiments on noisy images.
\end{enumerate}

%% file: sections/background.tex
\section{Background} \label{sec:background}

\subsection{Notation}\label{sec:notation}
Let $\N$ be the natural numbers, $\R$ be the real numbers, and $\B=\{0,1\}$.
We denote scalars and vectors by lowercase letters, matrices by uppercase letters, and sets by calligraphic letters.
The $i$-th element of a vector $v\in\R^n$ is written as $v_{(i)}$,
the element in the $i$-th row and $j$-th column of a matrix $A\in\R^{n\times m}$ is written as $A_{(i,j)}$,
and the entire $i$-th row and $j$-th column are written as $A_{(i,\cdot)}$ and $A_{(\cdot,j)}$, respectively.
We denote by $I_n$ the identity matrix of dimension $n\in\mathbb{N}$.
The bold symbols $\mathbf{0}$ and $\mathbf{1}$ refer to matrices with all zeros and ones of proper dimensions, respectively.
Given $n\in\mathbb{N}$, we use the shorthand notation $[n]=\left\{ 1,\ldots,n \right\}$.
The cardinality of a discrete set $\mathcal{D}$ is denoted by $|\mathcal{D}|$.
Let $\mathcal{D}\subseteq [n]$, then $A_{(\mathcal{D}, \cdot)}$ denotes all rows $i\in\mathcal{D}$ in lexicographic order;
this is used analogously for columns.
Given $a\leq b \leq n$, we also use $A_{(a:b,\cdot)}= A_{([b]\setminus[a-1],\cdot)}$ for readability when indexing.
Let $\mathcal{S}\subset\R^n$ be a set
and $f\colon \R^n\rightarrow \R^m$ be a function,
then $f(\mathcal{S}) = \left\{f(x)\ \middle|\ x \in \mathcal{S}\right\}$.
We also use the (matrix) indexing defined above for multi-dimensional sets.
An interval with bounds $a, b\in\R^n$ is denoted by $[a,b]$, where $a\leq b$ holds element-wise.
We denote the indicator function by $\indicator{\texttt{<condition>}}$.

\subsection{Model of the System}
\label{sec:system-model}

Our goal is to obtain a certified pose estimate of a target object from a binary image $\image\in\B^{w\times h}$ taken by an event-based camera, which is known for its sensitivity~\cite{gallego2020event}.
We assume that the target object can be decomposed into multiple convex polygons (\cref{fig:motivation-planar-target}),
and subsequently define a generalized form of polygons~\cite{de2008computational},
where their vertices are on a two-dimensional plane in a three-dimensional Euclidean space.
\begin{definition}[Polygon]
    \label{def:polygon}
    Given $c\in\R^{3}$, $d\in\R$ defining a plane,
    and $\numPGvertices$ vertices $V\in\R^{3\times \numPGvertices}$ in counter-clockwise winding order,
    we define a convex polygon as:
    \begin{align*}
        \PG = \shortPG{V} &= \opConvHull[inline]{V} \subset\R^3, \iftwocol \\ \else & \fi
        \text{with }c^\transpose V_{(\cdot,i)} &= d, \quad i\in[\numPGvertices],
    \end{align*}
    where \iftwocol \\ \else \fi $\opConvHull{V} = \left\{ \sum_{i=1}^{\numPGvertices} \lambda_i V_{(\cdot,i)} \ \middle|\ \lambda_i \in [0,1],\, \sum_{i=1}^{\numPGvertices} \lambda_i = 1 \right\}$.
\end{definition}
The target object is defined as follows:
\begin{definition}[Target Object]
    \label{def:target}
    Given $\numPGs$ convex polygons $\PG_1,\ldots,\PG_{\numPGs} \subset \R^3$,
    the target model $\target = \shortTarget{\PG_1,\ldots,\PG_\numPGs} \subset \R^3$ is:
    \begin{equation*}
        \target = \bigcup_{i=1}^\numPGs \PG_i.
    \end{equation*}
\end{definition}

The polygons of the target $\target$ are given in the \emph{target coordinate frame} (TCF).
This space is transformed through the camera (\emph{camera coordinate frame}, CCF) onto the final image (\emph{pixel coordinate frame}, PCF),
where we consider a pinhole camera model specified by intrinsic and extrinsic parameters~\cite[Chp.~6.1]{hartley2003multiple}.

The \emph{intrinsic parameters} consist of the focal length $\focalLen\in\R_{>0}$ as well as the width $w\in\N$ and height $h\in\N$ of the final image.
We assume that the principal point of the camera,
i.e., where the optical axis of the camera intersects the image sensor plane,
is at the center of the image.
The \emph{extrinsic parameters} involve the pose $\pose$ of the camera with respect to the target $\target$:
\begin{equation}
    \label{eq:pose}
    \pose = \cmatrix{x & y & z & \angleX & \angleY & \angleZ}^\transpose \in \poseSpace \subset \R^6,
\end{equation}
where $x,y,z$ is the position in the global coordinate frame,
$\angleX,\angleY,\angleZ$ are the roll, pitch, and yaw angles, respectively,
and $\poseSpace\subset\R^6$ is the pose space.

Subsequently, we detail how an image is computed from a pose $\pose$,
which we also illustrate in \cref{fig:system-model}:
Given a convex polygon $\PG_i=\shortPG{V_i^\text{TCF}}$ of a target $\target$, with $V_i^\text{TCF}\in\R^{3\times \numPGvertices_i}$,
and the parameters of the camera $\focalLen,w,h,\pose$,
the \emph{analog} output of the camera is~\cite[Chp.~6.1]{hartley2003multiple}:
\begin{equation}
    \label{eq:camera-analog}
    V_i^\text{CCF} = K_{\focalLen,w,h} \, \left( R(\angleX,\angleY,\angleZ) \, V_i^\text{TCF} + T(x,y,z) \right) \iftwocol \else \in\R^{3\times \numPGvertices_i} \fi,
\end{equation}
using the intrinsic matrix~\cite[Eq.~6.2]{hartley2003multiple}:
\begin{equation}
    \label{eq:intrinsic-params}
    K_{\focalLen,w,h} = \cmatrix{\focalLen & 0 & w/2 \\ 0 & \focalLen & h/2 \\ 0 & 0 & 1},
\end{equation}
the rotation matrix~\cite[Eq.~4.17]{goldstein1980classical}:
\begin{align}
    \label{eq:rot-mats}
    \begin{split}
        &R(\angleX,\angleY,\angleZ) = R_x(\angleX)\, R_y(\angleY)\, R_z(\angleZ) = \\
        &\fittocolwidth{
        \cmatrix{1 & 0 & 0 \\ 0 & \cos(\angleX) & -\sin(\angleX) \\ 0 & \sin(\angleX) & \cos(\angleX)}
        \cmatrix{ \cos(\angleY) & 0 & \sin(\angleY) \\ 0 & 1 & 0 \\   -\sin(\angleY) & 0 & \cos(\angleY)}
        \cmatrix{ \cos(\angleZ) & -\sin(\angleZ) & 0 \\ \sin(\angleZ) & \cos(\angleZ) & 0 \\ 0 & 0 & 1}
        },
    \end{split}
\end{align}
and the translation:
\begin{equation}
    T(x,y,z) = \cmatrix{x & y & z}^{\transpose}.
\end{equation}
The \emph{digital} behavior of the camera maps the polygon in the CCF onto the pixels of the binary image $\image_i\in\B^{w\times h}$.
This is computed by~\cite[Eq.~6.1]{hartley2003multiple}:
\renewcommand{\arraystretch}{1.5}
\begin{align}
    \begin{split}
        \label{eq:camera-digital}
        \forall (q_x,q_y) \in \TWOCOLAMP\  [w]\times[h]\colon\iftwocol \\ \else \quad \fi  \image_{i(q_x,q_y)} = &\begin{cases}
                    1 & \text{if $(q_x,q_y)\in\opConvHull[inline]{V_{i}^\text{PCF}}$,} \\
                    0 & \text{otherwise,}
        \end{cases} \\
    \text{where} \iftwocol \ \else \qquad \fi
        V_{i(\cdot,k)}^\text{PCF} = &\cmatrix{\sfrac{V_{i(1,k)}^\text{CCF}}{V_{i(3,k)}^\text{CCF}} \\ \sfrac{V_{i(2,k)}^\text{CCF}}{V_{i(3,k)}^\text{CCF}}} \in \R^{2}, \quad k\in [\numPGvertices_i].
    \end{split}
\end{align}

\begin{figure*}[t]
    \centering
    \includegraphics[alt={Visualization of a runway target transformation through all camera frames on the final image pixels}]{./figures/externalize\iftwocol -twocol\else \fi/figures/camera/pose-camera-model}
    \caption{
        Example transformation given a target $\target$ and a pose $\pose$.
        The goal of this work is the inverse:
        How to obtain a certified estimate of $\pose$ given the final image and $\target^\text{TCF}$?
    }
    \label{fig:system-model}
\end{figure*}

We now have all the ingredients to define the camera model used in this work.
\renewcommand{\arraystretch}{1}
\begin{definition}[Pinhole Camera Model~{\cite[Chp.~6.1]{hartley2003multiple}}]
    \label{def:camera}
    Given a target $\target = \{\PG_1,\ldots,\PG_\numPGs\}$, the pose $\pose\in\poseSpace$, and the parameters of the camera $\focalLen\in\R$, $w\in\N$, $h\in\N$,
    the images $\image_i\in\B^{w\times h}$ of each polygon $i\in[\numPGs]$, captured by a camera $\camera$, are computed using \crefrange{eq:camera-analog}{eq:camera-digital},
    and the combined image $\image\in\B^{w\times h}$ is given by:
    \begin{equation*}
        \image = \min\{1,\, \Sigma_{i=1}^\numPGs \image_i \}.
    \end{equation*}
\end{definition}
Thus, a pixel is turned on if any polygon $\PG_i$ intersects with the pixel.

\subsubsection*{Noisy images\iftwocol\else.\fi}

In practice, images are subject to noise for which we define a noise budget.
\begin{definition}[Noise Budget]
    \label{def:noisy-image}
    Given a clean image $I\in\B^{w\times h}$ as produced by \cref{def:camera},
    we say that an image $\imageNoisy\in\B^{w\times h}$ does not exceed
    a maximal noise budget $\imageNoiseLevel\in\N_0$ iff the number of perturbed pixels is bounded by:
    \begin{equation*}
        \lVert \image - \imageNoisy \rVert_1 \leq \imageNoiseLevel.
    \end{equation*}
\end{definition}
We assume that noise does not lead to turned-off pixels along the edges of the target,
as event-based cameras can reliably detect edges~\cite{gallego2020event}.

\subsection{Set-Based Computing}

We use (matrix) polynomial zonotopes~\cite{kochdumper2020sparse} to represent sensing uncertainties as their Minkowski sum and multiplication can be computed exactly and efficiently.
The definition and the required operations are provided together with an example in \cref{sec:pZ}.

\begin{definition}[Matrix Polynomial Zonotope~{\cite[Def.~7]{ladner2025formal}}]
    \label[definition]{def:pZ-mat}
    Given an offset $O\in\R^{n\times m}$,
    dependent generators $G \in \R^{n\times m \times h}$,
    independent generators $G_I \in\R^{n\times m\times q}$,
    and an exponent matrix $E\in \mathbb{N}_0^{p\times h}$ with an identifier $\PZid\in\mathbb{N}^p$,
    a matrix polynomial zonotope $\mathcal{M} = \shortPZ{O}{G}{G_I}{E}\subset\R^{n\times m}$ is defined as:
    \begin{align*}
        \mathcal{M} \coloneqq \left\{ O + \sum_{i=1}^h \left(\prod_{k=1}^p \alpha_k^{E_{(k,i)}}\right) G_{(\cdot,\cdot, i)} + \sum_{j=1}^{q} \beta_j G_{I(\cdot,\cdot, j)}\ \iftwocol \right| \\ \left. \else \middle|\ \fi \alpha_k, \beta_j \in \left[-1,1\right] \right\}.
    \end{align*}
\end{definition}

Given two matrix polynomial zonotopes $\mathcal{M}_1=\shortPZ{O_1}{G_1}{G_{I,1}}{E_1},\,\mathcal{M}_2=\shortPZ{O_2}{G_2}{G_{I,2}}{E_2}\subset\R^{n\times m}$ with a common identifier,
the Minkowski sum is computed as~\cite[Eq.~5]{ladner2025formal}:
\begin{align}
    \label{eq:PZ-mink-sum}
    \begin{split}
        \mathcal{M}_1\oplus\mathcal{M}_2 \AMP= \left\{ M_1 + M_2\ \middle|\ M_1\in\mathcal{M}_1,\,M_2\in\mathcal{M}_2 \right\} \\
        \AMP=\shortPZ{O_1 + O_2}{\cmatrix{G_1\ G_2}}{\cmatrix{G_{I,1}\ G_{I,2}}}{\cmatrix{E_1 & E_2 }} \iftwocol \else \subset\R^{n\times m} \fi.
    \end{split}
\end{align}
Given two matrix polynomial zonotopes $\mathcal{M}_3\subset\R^{n\times k}$, $\mathcal{M}_4\subset\R^{k\times m}$,
their multiplication:
\begin{align}
    \label{eq:PZ-mat-mul}
    \opQuadMapMat{\mathcal{M}_3}{\mathcal{M}_4} = \left\{ (M_3 M_4)\ \middle|\ M_3\in\mathcal{M}_3,\, M_4\in\mathcal{M}_4 \right\}
\end{align}
is computed by broadcasting the multiplication onto each offset and generator combination~\cite[Lemma~1]{ladner2025formal}.
Notably, an affine map is computed as a special case of \cref{eq:PZ-mink-sum} and \cref{eq:PZ-mat-mul},
with one matrix polynomial zonotope being a singleton.
Subscript operations on a matrix polynomial zonotope $\PZ=\shortPZ{O}{G}{G_I}{E}\subset\R^{n\times m}$
can be computed by applying them to $O$, $G$, $G_I$.
For example, the uncertainty at the $(i,j)$-th entry of the matrix, $i\in[n],j\in[m]$, is obtained by:
\begin{equation}
    \PZ_{(i,j)} = \shortPZ{O_{(i,j)}}{G_{(i,j,\cdot)}}{G_{I(i,j,\cdot)}}{E}.
\end{equation}
We use matrix polynomial zonotopes to bound the uncertain vertices of a polygon $\V\subset\R^{3\times\numPGvertices}$,
and for $k\in[\numPGvertices]$, $\V_{(\cdot,k)}\subset\R^3$ returns the uncertainty of the $k$-th vertex.

Given a set $\mathcal{X}\subset\R^n$ and a function $f\colon\R^n\rightarrow\R^m$,
we also need to bound $f(\mathcal{X}) \subseteq \mathcal{Y} \subset\R^m$.
To do so, we leverage recent advances in formal neural network verification~\cite{kaulen2025vnn},
in particular, for polynomial zonotopes~\cite{kochdumper2022open,ladner2023automatic}.
Verifying certain properties of the output of $f$ given $\mathcal{X}$ is NP-hard~\cite{katz2017reluplex}.
Thus, the output is usually enclosed using abstractions,
for which polynomial zonotopes have a good trade-off between precision and computation time.
\begin{proposition}[Image Enclosure~{\cite[Sec. 3]{kochdumper2022open}}]
    \label[proposition]{prop:image-enclosure}
    Given a set $\mathcal{X}\subset\R^n$ and a function $f\colon\R^n\rightarrow\R^m$,
    \begin{equation*}
        \mathcal{Y} = \opEnclose{f}{\mathcal{X}} \supseteq f(\mathcal{X})
    \end{equation*}
    computes an outer-approximative output set.
    We provide details on $\opEnclose{\cdot}{\cdot}$ in \cref{sec:pZ}.
\end{proposition}

\subsection{Problem Statement}
\label{sec:problem-statement}

Given a target $\target\subset\R^3$ (\cref{def:target}), a camera $\camera$ (\cref{def:camera}), and a binary image $\image^*\in\B^{w\times h}$,
we want to find a certified set of poses $\poseUncertain^*\subseteq\poseSpace$
such that the (unobservable) true pose $\pose^*\in\poseSpace$,
which reproduces $\image^*$ up to a considered noise budget $\imageNoiseLevel$,
is contained: $\pose^*\in\poseUncertain^*$.
To penalize simply returning the trivial solution $\poseUncertain^*=\poseSpace$,
pose estimates $\poseUncertain^*$ with a smaller volume are preferred.

%% file: sections/content.tex
\section{Framework for Certified Pose Estimation} \label{sec:content}

Our approach to obtain a certified pose estimate is based on reachability analysis and preimage enclosures using the computed reachable sets.
In particular, we compute outer-approximative images from a set of pose candidates in \cref{sec:pose-to-image},
on which we later, given a concrete image, impose constraints to obtain a tight certified pose estimate in \cref{sec:image-to-pose}.
An overview of this process is shown in \cref{fig:overview}.

\begin{figure*}
    \centering
    \includegraphics[width=\linewidth,alt={Interaction of components of our approach.}]{./figures/externalize\iftwocol -twocol\else \fi/figures/overview/overview}
    \caption{
        Overview: For each pose candidate $\poseUncertain_c$ (\inlineLegendCube[0]{black}), we pre-compute the possible positions of a target $\target$ in the camera image to obtain an outer-approximative image $\imageOuterApprox_{\poseUncertain_c}$ (\inlineLegendBox{black}).
        Subsequently, we check online whether (a) a given concrete image $\image^*$ (\inlineLegendBox{black}) is contained in $\imageOuterApprox_{\poseUncertain_c}$.
        (b) If not, $\poseUncertain_c$ is discarded since $\pose^*\not\in\poseUncertain_c$. Otherwise, we exploit the relative position of the vertices in $\image^*$ (\inlineLegendBox{CORAcolorYellow}) within the uncertain vertices (\inlineLegendBox{CORAcolorLightBlue}) of $\poseUncertain_c$
        -- which are a byproduct of computing $\imageOuterApprox_{\poseUncertain_c}$ --
        to (c) impose constraints (\inlineLegendBox{CORAcolorGreen}) on $\poseUncertain_c$, which -- when computed for all vertices -- obtains the certified pose estimate $\poseUncertain^*$ (\inlineLegendBox{CORAcolorRed}).
    }
    \label{fig:overview}
\end{figure*}

\subsection{Sound Image Enclosure From Uncertain Pose}
\label{sec:pose-to-image}

Given an uncertain pose $\poseUncertain\subseteq\poseSpace$,
we present an approach to enclose all resulting images from concrete poses $\pose\in\poseUncertain$.
This section closely follows \cref{sec:system-model}, while computing everything in a set-based manner in \cref{alg:pose-to-image}.

\begin{algorithm}[h!]
    \caption{Enclosing an Image From Uncertain Pose}
    \label{alg:pose-to-image}
    \begin{algorithmic}[1]
        \Require {Target $\target=\shortPG{\PG_1,\ldots,\PG_\numPGs}\in\R^3$, uncertain pose $\poseUncertain\subseteq\poseSpace$, camera $\camera$}
        \State \texttt{// Analog behavior of camera.} \label{alg-line:analog-start}
        \State $\contSet{S} \gets \opEnclose[inline]{\sin}{\poseUncertain_{(4:6)}}$ \label{alg-line:sin-cos} \Comment{\cref{prop:image-enclosure}}
        \State $\contSet{C} \gets \opEnclose[inline]{\cos}{\poseUncertain_{(4:6)}}$
        \State Construct $\contSet{R}_x,\contSet{R}_y,\contSet{R}_z$ using $\contSet{S},\contSet{C}$ \label{alg-line:rot-mats} \Comment{\cref{eq:rot-mats}}
        \State $\contSet{R} \gets \opQuadMapMat{\contSet{R}_x}{\opQuadMapMat{\contSet{R}_y}{\contSet{R}_z}}$ \Comment{\cref{eq:rot-mats}, \cref{eq:PZ-mat-mul}}
        \For{$\PG_i =\shortPG{V_i^\text{TCF}}$ \textbf{in} $\target$}
            \State $\V_i^\text{CCF} \gets \opQuadMapMat{K_{\focalLen,w,h}}{(\opQuadMapMat{\contSet{R}}{V_i^\text{TCF}} \oplus \poseUncertain_{(1:3)})}$  \label{alg-line:analog-end} \Comment{\cref{eq:camera-analog-uncertain}}
            \State \texttt{// Digital behavior of camera.}
            \State $\V_{i(1:2,\cdot)}^\text{PCF} \gets \opQuadMapMat{\V_{i(1:2,\cdot)}^\text{CCF}}{\opEnclose[inline]{x \mapsto \sfrac{1}{x}}{\V_{i(3,\cdot)}^\text{CCF}}}$ \Comment{\cref{eq:camera-digital}\iftwocol \else, \cref{prop:image-enclosure}\fi}
            \State $\widehat{\PG}_i^\text{PCF} \supseteq \opConvHull{\V_{i(\cdot,1)}^\text{PCF},\ldots,\V_{i(\cdot,\numPGvertices_i)}^\text{PCF}}$ \label{alg-line:convHull} \Comment{\cref{eq:convHull}}
            \State $\imageOuterApprox_i \gets \mathbf{0}\in\B^{w \times h}$ \label{alg-line:pixel-contains-start} \Comment{Initialize image\ }
            \For{$(q_x,q_y) \in [w]\times[h]$}
                \State $\imageOuterApprox_{i(q_x,q_y)} \gets \indicator{(q_x,q_y)\in\widehat{\PG}_i^\text{PCF}}$ \label{alg-line:pixel-inclusion}\Comment{Check \iftwocol \else pixel \fi containment}
            \EndFor  \label{alg-line:pixel-contains-end}
        \EndFor
        \State \Return \iftwocol \else  Outer-approx.\ image \fi $\imageOuterApprox_\poseUncertain \gets \min\{1,\, \sum_{i=1}^\numPGs \imageOuterApprox_i \}$ \label{alg-line:combine-polygon-images} \Comment{\cref{def:camera}}
    \end{algorithmic}
\end{algorithm}

We first enclose the $\sin$ and $\cos$ functions for the given angles $\poseUncertain_{(4:6)}$ using \cref{prop:image-enclosure} (line~\ref{alg-line:sin-cos}),
and bound all possible rotation matrices $\contSet{R}$ (\cref{eq:rot-mats}) using \cref{eq:PZ-mat-mul} (line~\ref{alg-line:rot-mats}).
Subsequently, the analog output of the camera (\cref{eq:camera-analog}) is computed by:
\begin{equation}
    \label{eq:camera-analog-uncertain}
    \V_i^\text{CCF} = \opQuadMapMat{K_{\focalLen,w,h}}{(\opQuadMapMat{\contSet{R}}{V_i^\text{TCF}} \oplus \poseUncertain_{(1:3)})}
\end{equation}
using \cref{eq:PZ-mat-mul} and \cref{eq:PZ-mink-sum}.
Please note that uncertainty in both the camera parameters and the target geometry can be modeled by adding uncertainty to $K_{\focalLen,w,h}$ and $V^\text{TCF}$, respectively;
further details are provided in \cref{sec:uncertain-target-geometry}.

\begin{figure*}[t]
    \centering
    \includegraphics[alt={Step-by-step construction of an outer-approximative image}]{./figures/externalize\iftwocol -twocol\else \fi/figures/supportFunc/pose-convHull}
    \caption{
        Computation of the outer-approximative image for a target with a single polygon ($\numPGs=1$) and an uncertain pose $\poseUncertain$ with perturbed angles $\angleX,\angleY,\angleZ$ by $10$ degrees:
        (a) Uncertain vertices $\V_{1}^\text{PCF}$ enclose the vertices of the random samples.
        (b) Convex hull computation over $\V_{1}^\text{PCF}$.
        (c) Outer-approximative image $\imageOuterApprox_\poseUncertain$.
    }
    \label{fig:convHull}
\end{figure*}

The digital behavior of the camera (\cref{eq:camera-digital}) requires us to compute the convex hull over all vertices $\V_{i(\cdot,k)}^\text{PCF}$ (\cref{def:polygon});
however, these are now sets themselves due to the uncertainty in the pose $\poseUncertain$ (\cref{fig:convHull}).
As $\V_{i(\cdot,k)}^\text{PCF}$ is two-dimensional, efficient algorithms exist to compute this convex hull~\cite{bronnimann2006space}:
\begin{equation}
    \label{eq:convHull}
    \widehat{\PG}_i^\text{PCF} \supseteq \opConvHull{\V_{i(\cdot,1)}^\text{PCF},\ldots,\V_{i(\cdot,\numPGvertices_i)}^\text{PCF}},
\end{equation}
and more details are provided in \cref{sec:convHull}.
Finally, we check for pixel containment in the outer-approximative polygon $\widehat{\PG}_i^\text{PCF}$ of each polygon $i\in[\numPGs]$ (lines~\ref{alg-line:pixel-contains-start}-\ref{alg-line:pixel-contains-end}),
and combine the resulting images to obtain the final image (line~\ref{alg-line:combine-polygon-images}).
This final image is used later as an initial filter of the pose candidates.

\begin{restatable}[Pose to Image]{proposition}{propalgForward}
    \label{prop:algForward}
    Given a target $\target=\shortPG{\PG_1,\ldots,\PG_\numPGs}\in\R^3$, an uncertain pose $\poseUncertain\subseteq\poseSpace$, and a camera $\camera$,
    \cref{alg:pose-to-image} computes an image $\imageOuterApprox_\poseUncertain\in\B^{w\times h}$ such that:
    \begin{align*}
        \forall \pose\in\poseUncertain\colon &\forall (q_x,q_y) \in [w]\times[h]\colon\iftwocol \\\TWOCOLAMP \else \fi \qquad \indicator{\image_{\pose(q_x,q_y)} = 1}\ \implies\  \indicator{\imageOuterApprox_{\poseUncertain(q_x,q_y)}=1},
    \end{align*}
    where $\image_{\pose}\in\B^{w\times h}$ is the image produced by a concrete pose $\pose$ using \cref{def:camera}.
    The computational complexity is $\mathcal{O}(\numPGs wh)$.
\end{restatable}
\begin{proof}
    The proof can be found in \cref{sec:proofs}. \iftwocol \qedhere \else \hfill \qed \fi
\end{proof}

\subsection{Certified Pose Estimation From Image}
\label{sec:image-to-pose}

Now that we have computed the outer-approximative image $\imageOuterApprox_\poseUncertain$, we detail how to obtain a certified pose estimate given a concrete, possibly noisy, image $\image^*$ as shown in \cref{fig:overview}.
Our approach exploits dependencies between the computed sets, i.e., between the uncertain pose $\poseUncertain\subseteq\poseSpace$ and the vertices $\V_{i(\cdot,k)}^\text{PCF}$ computed via \cref{alg:pose-to-image}.
In particular, these sets have a \emph{shared latent space}, 
enabling us to impose constraints on $\poseUncertain$ from constraints identified on $\V_{i(\cdot,k)}^\text{PCF}$.
This constrained $\poseUncertain$ will then be our certified pose estimate $\poseUncertain^*\subseteq\poseSpace$ ensuring that the (unobservable) true pose $\pose^*\in\poseSpace$ is contained: $\pose^*\in\poseUncertain^*$ (\cref{sec:problem-statement}).
Subsequently, we detail the individual steps of our approach performed in \cref{alg:image-to-pose}:
1) the selection of candidate poses $\poseUncertain_c$,
and 2) the refinement of the candidate poses by imposing constraints on the shared latent space.

\begin{algorithm}
    \caption{Obtaining Certified Pose Estimates From Image}
    \label{alg:image-to-pose}
    \newcommand{\IfInline}[2]{\textbf{if} #1 \textbf{then} #2}
    \newcommand{\algorithmiccontinue}{\textbf{continue}}
    \begin{algorithmic}[1]
        \Require {Image $\image^*\in\B^{w\times h}$, target $\target\subset\R^3$, camera $\camera$}
        \State \texttt{// 1) \iftwocol Filter \else Gather and filter \fi candidate poses.}
        \State Gather candidate poses $\poseUncertain_1,\ldots,\poseUncertain_\numPoseCandidates\subseteq\poseSpace$ \Comment{(pre-computed)}
        \For{$c\in[\numPoseCandidates]$}
            \State Gather $\imageOuterApprox_{\poseUncertain_c},\V_{c,1}^\text{PCF},\ldots,\V_{c,\numPGs}^\text{PCF}$ \iftwocol \else for $\poseUncertain_c$ \fi \Comment{\cref{alg:pose-to-image} (pre-computed)}
            \If{$\image^*\not\in\imageOuterApprox_{\poseUncertain_c}$}
                \State $C_c\gets \mathbf{0}$ ; $d_c\gets-1$ ; \algorithmiccontinue \label{alg-line:quick-exit} \Comment{Quick exit (\cref{eq:image-inclusion-check})}
            \EndIf
            \State \texttt{// 2) Refine candidate poses.}
            \For{$i\in[\numPGs]$}
                \For{$k\in[\numPGvertices_i]$}
                    \State Compute \iftwocol \else witness pixels \fi $\Q_{c,i,k}\subseteq[w]\times[h]$ for $\V_{c,i(\cdot,k)}^\text{PCF}$ \Comment{\cref{eq:vertex-certified-estimate}}
                    \State Compute $C_{c,i,k},d_{c,i,k}$ \iftwocol \else using $\V_{c,i(\cdot,k)}^\text{PCF}$, $\Q_{c,i,k}$ \fi
                    \label{alg-line:constraints-via-crit-pixel}\Comment{\cref{eq:vertex-certified-estimate}, \cref{prop:enclosing-unsafe-inputs}}
                \EndFor
                \State \iftwocol \else Collect constraints \fi $C_{c,i} = \cmatrix{C_{c,i,1}^\transpose & \ldots & C_{c,i,\numPGvertices_i}^\transpose}^\transpose$ \iftwocol
                \State \else ; \fi $d_{c,i} = \cmatrix{d_{c,i,1}^\transpose & \ldots & d_{c,i,\numPGvertices_i}^\transpose}^\transpose$ \label{alg-line:collect-1}
            \EndFor
            \State \iftwocol \else Collect constraints \fi $C_{c} = \cmatrix{C_{c,1}^\transpose & \ldots & C_{c,\numPGs}^\transpose}^\transpose$ ; $d_{c} = \cmatrix{d_{c,1}^\transpose & \ldots & d_{c,\numPGs}^\transpose}^\transpose$ \label{alg-line:collect-2}
        \EndFor
        \State \Return \iftwocol \else Certified pose estimate \fi $\poseUncertain^* \gets \cZ{\poseUncertain_1}{C_1}{d_1}\cup\, \cdots\, \cup\,\cZ{\poseUncertain_\numPoseCandidates}{C_\numPoseCandidates}{d_\numPoseCandidates}\subseteq\poseSpace$ \Comment{\cref{eq:certified-pose-estimate}}
    \end{algorithmic}
\end{algorithm}

\subsubsection{\iftwocol\else1) \fi Selection of candidate poses\iftwocol\else.\fi}
We can pre-compute a set of (axis-aligned) boxes $\poseUncertain_1,\ldots,\poseUncertain_\numPoseCandidates$
partitioning the pose space $\poseSpace\subset\R^6$ offline:
\begin{equation}
    \label{eq:pose-space-partitioning}
    \poseSpace = \bigcup_{c\in[\numPoseCandidates]} \poseUncertain_c,
\end{equation}
and compute the corresponding $\imageOuterApprox_{\poseUncertain_c}\in\B^{w\times h}$, $\V_{c,1}^\text{PCF},\ldots,\V_{c,\numPGs}^\text{PCF}\subset\R^2$ for each candidate pose $\poseUncertain_c$, $c\in[\numPoseCandidates]$, using \cref{alg:pose-to-image}.
Then, online, e.g., as a landing plane visually observes the runway (\cref{fig:motivation-planar-target}),
we efficiently filter all irrelevant boxes $\poseUncertain_c$ (\cref{alg:image-to-pose}, line~\ref{alg-line:quick-exit}) given a concrete image $\image^*\in\B^{w\times h}$ by checking:
\begin{align}
    \label{eq:image-inclusion-check}
    \begin{split}
    \image^* \in \imageOuterApprox_{\poseUncertain_c}, \qquad
    \text{i.e., } \TWOCOLAMP  \forall (q_x,q_y) \in [w]\times[h]\colon  \iftwocol \\ \else \fi
        \TWOCOLAMP \indicator{\image^*_{(q_x,q_y)} = 1}\ \implies\  \indicator{\imageOuterApprox_{\poseUncertain_c(q_x,q_y)}=1}.
    \end{split}
\end{align}
For the remaining pose candidates, we analogously check if $\image^*$ is at least partially contained in $\imageOuterApprox_{\poseUncertain_c,i,k}$ for each vertex $k\in[\numPGvertices_i]$, $i\in[\numPGs]$.
Due to the outer-approximative computation of the images, we can guarantee that the true pose $\pose^*$ has to be contained in the remaining pose candidates.
The filtering of candidates has to be implemented efficiently to be computationally feasible online.
Since all involved variables are binary bitmaps, this can be done batch-wise on all $\numPoseCandidates$ candidates with complexity $\mathcal{O}(\numPoseCandidates \numPGs \numPGvertices_i w h)$.

\subsubsection{\iftwocol\else2) \fi Refinement of candidate poses\iftwocol\else.\fi}
We refine our selected candidate poses $\poseUncertain_c$,
by imposing constraints $C_c \leq d_c$ on the latent space shared between $\poseUncertain_c$
and its computed vertices $\V_{c,i(\cdot,k)}^\text{PCF}$.
In particular, let $\poseUncertain_c=\shortPZ{o_c}{G_c}{[\ ]}{I_6} \subseteq \poseSpace$ be an axis-aligned box,
and let $\hypercube_c = [-\mathbf{1},\mathbf{1}] \subset \R^6$ be a hypercube,
then $\poseUncertain_c$ is simply a projection of that hypercube $\hypercube_c$ (compare with \cref{def:pZ-mat}):
\begin{equation}
    \label{eq:hypercube-pose}
    \poseUncertain_c = o_c + G_c\hypercube_c.
\end{equation}
Crucially, as each vertex $\V_{c,i(\cdot,k)}$ is computed through \cref{alg:pose-to-image} from $\poseUncertain_c$,
and, thus, ultimately from $\hypercube_c$ (\cref{eq:hypercube-pose}),
it can be decomposed as follows:

\begin{align}
    \TWOCOLAMP \V_{c,i(\cdot,k)}^\text{PCF} \AMP= \shortPZ{\widetilde{o}_{c,i,k}}{[\widetilde{G}_{c,i,k}\ \widehat{G}_{c,i,k}]}{\widehat{G}_I}{[I_6\ \widehat{E}_{c,i,k}]} \nonumber \\
    &=  \underbrace{\shortPZ{\widetilde{o}_{c,i,k}}{\widetilde{G}_{c,i,k}}{[\ ]}{I_6}}_\text{Linearized term $ \V_{c,i(\cdot,k)}^\text{LIN}$} \oplus \underbrace{\shortPZ{\mathbf{0}}{\widehat{G}_{c,i,k}}{\widehat{G}_I}{\widehat{E}_{c,i,k}}}_\text{Linearization error $\V_{c,i(\cdot,k)}^\text{ERR}$}, \iftwocol  \nonumber \else \label{eq:hypercube-vertex} \fi \\
    \TWOCOLAMP \text{with } \V_{c,i(\cdot,k)}^\text{LIN} \AMP= \widetilde{o}_{c,i,k} + \widetilde{G}_{c,i,k} \hypercube_c. \iftwocol  \label{eq:hypercube-vertex} \else \nonumber \fi
\end{align}
Thus, the linearized term of $\V_{c,i(\cdot,k)}^\text{PCF}$ is a projection of the hypercube $\hypercube_c$ and
$\hypercube_c$ \emph{is} the shared latent space between $\V_{c,i(\cdot,k)}^\text{PCF}$ and $\poseUncertain_c$.
We utilize this to impose constraints from the output space on the uncertain input to enclose a preimage of the function computed by \cref{alg:pose-to-image} (mapping from the pose space $\poseSpace$ to the PCF space in $\R^2$):
\begin{restatable}[{Preimage Enclosure~\cite[Prop.~2]{koller2025out}}]{proposition}{propenclosingunsafeinputs}
    \label{prop:enclosing-unsafe-inputs}
    Given a function $f\colon\R^n\rightarrow\R^m$,
    an input set $\contSet{X}=\shortPZ{o_x}{G_x}{[\ ]}{I_n}\subset\R^n$,
    and a polytope $\contSet{U}=\shortPH{A}{b}\subset\R^m$,
    let $\contSet{Y} = \shortPZ{\widetilde{o}_y}{[\widetilde{G}_y\ \widehat{G}_y]}{G_I}{[I_n\ \widehat{E}_y]} = \opEnclose{f}{\contSet{X}}\subset\R^m$ and let $f_\contSet{X}^{-1}(\contSet{U}) =\{x\in\contSet{X}\ |\ f(x)\in\contSet{U}\}$.
    Then,
    \begin{align*}
        f_\contSet{X}^{-1}(\contSet{U}) \subseteq \cZ{\contSet{X}}{C}{d} = \left\{ o_x + \sum_{i=1}^n \alpha_{(i)} G_{x(\cdot, i)}\ \iftwocol \right| \\ \left. \else \middle|\ \fi \alpha\in\hypercube=[-\mathbf{1},\mathbf{1}],\, C\hypercube\leq d  \right\},
    \end{align*}
    \begin{align*}
        \text{where} & \qquad C = A\widetilde{G}_y, \qquad d = b - A\widetilde{o}_y + |A \widehat{G}_y|\mathbf{1}.
    \end{align*}
    The preimage enclosure $\cZ{\contSet{X}}{C}{d}$ is effectively a constrained zonotope~\cite[Def.~3]{scott2016constrained},
    except that they are defined with equality constraints in the original paper.
\end{restatable}
\begin{proof}
    The proof can be found in \cref{sec:proofs}. \iftwocol \qedhere \else \hfill \qed \fi
\end{proof}

Next, we detail how \cref{prop:enclosing-unsafe-inputs} can be used to impose the constraints and provide a running example in \cref{fig:pose-hypercube}.
Intuitively, the set $\contSet{U}$ must contain the (unobservable) vertex $V_{i(\cdot,k)}^*$ produced by $\pose^*$ such that the preimage enclosure contains the true pose $\pose^*\in\poseSpace$ (\cref{sec:problem-statement}).
We construct $\contSet{U}$ by choosing pixels $\Q_{c,i,k}^*\subseteq[w]\times[h]$ from the given image $\image^*$
such that $V_{i(\cdot,k)}^*$ is contained.
After defining the discretization error $\discError=[-\mathbf{1},\mathbf{1}]/2\subset\R^2$ from the image creation (\cref{eq:camera-digital}), we have that:
\begin{equation}
    \label{eq:vertex-certified-estimate}
    V_{i(\cdot,k)}^* \in \V_{c,i(\cdot,k)}^* = \Q_{c,i,k}^* \oplus \discError \subset\R^2.
\end{equation}
Please note that there is ambiguity in the pixels that could contain $V_{i(\cdot,k)}^*$ due to the surjection during the image creation (\cref{def:camera}).
For the remainder of this section, we add all turned-on pixels in $\V_{c,i(\cdot,k)}$ to $\Q_{c,i,k}^*$,
which trivially fulfills this requirement (\cref{eq:vertex-certified-estimate}) as $V_{i(\cdot,k)}^*\in\V_{c,i(\cdot,k)}$ must hold due to the outer-approximative computation (\cref{prop:algForward}).
As $\Q_{c,i,k}^*$ is constructed from the observed image $\image^*$, we refer to its elements as \emph{witness pixels} (\cref{fig:pose-hypercube}a).
We further discuss the ambiguity and possible removals from the set of witness pixels $\Q^*_{c,i,k}$ in \cref{sec:crit-pixel-ambiguity}.

\begin{figure*}
    \centering
    \includegraphics[alt={Identification of witness pixels on a given image to impose constraints on the shared hypercube}]{./figures/externalize\iftwocol -twocol\else \fi/figures/hypercube/pose-hypercube}
    \caption{
        Certified pose estimation (\cref{alg:image-to-pose}): (a) Given an image $\image^*$ and a (simple) target $\target^\text{TCF}$ ($\numPGs=1$), 
        (b) the witness pixels $\Q_{c,1,k}$ are contained in the uncertain vertices $\V_{c,i(\cdot,k)}^\text{PCF}$ of the (unobservable) target $\target^\text{PCF}$, $k\in[\numPGvertices]$,
        enabling us to impose constraints $C_{c,1,k}\leq d_{c,1,k}$ on a pose candidate $\poseUncertain_c\subseteq\poseSpace$ via the shared latent space to (c) obtain a tight certified estimate of $\pose^*\in\cZ{\poseUncertain_c}{C_c}{d_c}$.
    }
    \label{fig:pose-hypercube}
\end{figure*}

Finally, to apply \cref{prop:enclosing-unsafe-inputs},
we compute $\contSet{U}_{c,i,k}\supseteq\V_{c,i(\cdot,k)}^*$ using \cref{eq:convHull},
and impose the resulting constraints $C_{c,i,k}$, $d_{c,i,k}$ on $\poseUncertain_c$  (\cref{fig:pose-hypercube}b).
Crucially, all vertices $k\in[\numPGvertices_i]$ share the \emph{same} latent space, and $\pose^*$ must be contained in the intersection of all constraints (\cref{fig:pose-hypercube}c).
Thus, for each uncertain pose $\poseUncertain_c$, $c\in[\numPoseCandidates]$,
we collect the constraints of all vertices $C_c$, $d_c$ (\cref{alg:image-to-pose}, lines \ref{alg-line:collect-1} to \ref{alg-line:collect-2}),
and the final certified pose estimate is given by:
\begin{equation}
    \label{eq:certified-pose-estimate}
    \poseUncertain^* = \cZ{\poseUncertain_1}{C_1}{d_1} \cup\ \cdots\ \cup\ \cZ{\poseUncertain_\numPoseCandidates}{C_\numPoseCandidates}{d_\numPoseCandidates}\subseteq\poseSpace.
\end{equation}
This process can again be computed batch-wise to be feasible online.
Please note that in \cref{alg:image-to-pose}, if an irrelevant candidate pose $\poseUncertain_c$ is filtered in phase 1,
we set the constraint $C_c=\mathbf{0}$ and $d_c=-1$ to indicate that $\pose^* \notin \poseUncertain_c$ as the constraint is infeasible.

\begin{restatable}[Image to Pose]{theorem}{thalgBackward}
    \label{th:algBackward}
    Given a target $\target\subset\R^3$ (\cref{def:target}), a camera $\camera$ (\cref{def:camera}), an image $\image^*\in\B^{w\times h}$ produced by an (unobservable) true pose $\pose^*\in\poseSpace$,
    and $\numPoseCandidates$ pose candidates partitioning $\poseSpace$,
    \cref{alg:image-to-pose} computes a certified pose estimate:
    \begin{equation*}
        \poseUncertain^*\subseteq\poseSpace \quad \text{s.t.}\quad \pose^*\in\poseUncertain^*
    \end{equation*}
    with a computational complexity of $\mathcal{O}(\numPoseCandidates\numPGs\numPGvertices_{i} wh)$.
\end{restatable}
\begin{proof}
    The proof can be found in \cref{sec:proofs}. \iftwocol \qedhere \else \hfill \qed \fi
\end{proof}

\section{Witness Pixels and Where to Find Them}
\label{sec:crit-pixel-ambiguity}

In this section, we consider reducing the set of witness pixels $\Q_{c,i,k}^*\subseteq[w]\times[h]$ (\cref{eq:vertex-certified-estimate}).
Please note that all other components of \cref{alg:image-to-pose} remain identical.
Any removal from $\Q_{c,i,k}^*$ imposes stricter constraints on the pose candidate $\poseUncertain_c$ (\cref{prop:enclosing-unsafe-inputs}),
thereby obtaining a tighter certified pose estimate $\poseUncertain^*$.
We begin by discussing the ambiguity due to the surjection in the image generation (\cref{def:camera}) in more detail.
In particular, there is ambiguity regarding which pixels contain the (unobservable) vertices $V_{i(\cdot,k)}^*$ of the true pose $\pose^*$.
This ambiguity is exemplarily illustrated in \cref{fig:ambiguity_crit_pixels}a-c:
Different poses $\pose^{(a)},\pose^{(b)},\pose^{(c)}\in\poseSpace$ result in different positions of $V_{i(\cdot,k)}^*$ of the transformed target $\target^\text{PCF}$,
but all result in the same image $\image^*$.

\iftwocol \begin{figure*}[!b] \else \begin{figure*}[t] \fi
    \centering
    \includegraphics[alt={Visualization of how different target positions can lead to the same image}]{./figures/externalize\iftwocol -twocol\else \fi/figures/ambiguity_crit_pixels/ambiguity_crit_pixels}
    \caption{
        Ambiguity of an (unobservable) target $\target^\text{PCF}$ given an image $\image^*$ (shown for $\imageNoiseLevel=0$ in zoomed-in view):
        (a-c) Possible orientations of $\target^\text{PCF}$ resulting in the same image $\image^*$.
        (d) Witness pixels must be on the boundary of $\image^*$ given the pre-computed $\V_{c,i(\cdot,k)}^\text{PCF}$.
    }
    \label{fig:ambiguity_crit_pixels}
\end{figure*}

To avoid additional ambiguity, we only apply this tightening on \emph{standalone} vertices $k$,
i.e., vertices for which their uncertain set $\V_{c,i(\cdot,k)}^\text{PCF}$ does not overlap with other vertices of the same polygon or other polygons altogether:
  \begin{align}
      \label{eq:vertex-no-overlap}
      \begin{split}
          \forall k'\in[\numPGvertices_i]\setminus\{k\}\colon\qquad &\V_{c,i(\cdot,k)}^\text{PCF} \cap \V_{c,i(\cdot,k')}^\text{PCF} = \emptyset, \\
          \text{and } \forall i'\in[\numPGs]\setminus\{i\}\colon\qquad &\V_{c,i(\cdot,k)}^\text{PCF} \cap \widehat{\contSet{P}}_{c,i'}^\text{PCF} = \emptyset.
      \end{split}
  \end{align}
  Thus, we can look at each remaining vertex individually and do not consider edge cases involving overlaps with other parts of the target%
  .
Using \cref{eq:vertex-no-overlap} and the convexity of each polygon of $\target^\text{PCF}$ (\cref{def:target}),
we can infer that $V_{i(\cdot,k)}^*$ has to be on the boundary of the restricted region, i.e., the pixels within $\V_{c,i(\cdot,k)}^\text{PCF}$ (\cref{fig:ambiguity_crit_pixels}d);
however, the noise makes this boundary fuzzy:
Please note that in our noise threat model (\cref{def:noisy-image}), we assume that the edges are reliably detected by the event-based camera~\cite{gallego2020event},
thereby guaranteeing that $V_{i(\cdot,k)}^*$ is within a turned-on pixel.
However, there might be additional turned-on pixels outside the target and some pixels inside the target might be turned off (\cref{fig:ambiguity_crit_pixels2}a-b).
Conversely, the latter case helps to impose tighter constraints directly as these pixels are not added to $\Q_{c,i,k}^*$ in the first place.

It is also possible to restrict the ambiguous pixels by reasoning geometrically.
As we optimize for speed in the online setting,
we only apply a basic geometric optimization for a noise budget $\imageNoiseLevel=0$ (\cref{def:noisy-image}),
and leave more complex optimizations for future work.
As all remaining witness pixels $\Q_{c,i,k}^*$ on the boundary are turned on,
there has to exist an edge of $\target^\text{PCF}$ intersecting with these pixels.
In particular, this edge must cross the left- and right-most witness pixels with respect to the direction orthogonal to the boundary.
Thus, any triangle spanned by a witness pixel $q\in\Q_{c,i,k}^*$ and the left- and right-most witness pixel,
expanded to cover the entire pixel (\cref{eq:vertex-certified-estimate}),
has to intersect with all remaining witness pixels $q'\in\Q_{c,i,k}^*$ (\cref{fig:ambiguity_crit_pixels2}c).
Otherwise, $q$ cannot contain the vertex $V_{i(\cdot,k)}^*$ of the true pose $\pose^*$ as $q'$ would not be turned on then.
As a special case of this geometric optimization, if a pixel $q\in\Q_{c,i,k}^*$ only has one turned-on neighboring pixel,
$q$ is the only witness pixel, as such a triangle cannot be spanned using another pixel that intersects with the selected pixel $q$.

\iftwocol \begin{figure*}[t] \else \begin{figure*}[t] \fi
    \centering
    \includegraphics[alt={Example of a noisy image and how geometric reasoning can restrict the number of considered pixels.}]{./figures/externalize\iftwocol -twocol\else \fi/figures/ambiguity_crit_pixels/ambiguity_crit_pixels2}
    \caption{Noise and geometry: (a) Idealized image. (b) Image with noise, where we assume that event-based cameras can reliably detect edges; thus, these are not affected by noise. (c) Triangle construction of the selected witness pixels~({\color{CORAcolorPurple} $\mathbf{\circ}$}), without pixel expansion (\cref{eq:vertex-certified-estimate}).}
    \label{fig:ambiguity_crit_pixels2}
\end{figure*}

Future work could also consider loosening \cref{eq:vertex-no-overlap},
e.g., by applying a similar technique on inner corners of a target,
which are currently ignored due to the second condition of \cref{eq:vertex-no-overlap}.
This could, e.g., be useful on the inner corners of the digit \texttt{0} in \cref{fig:system-model}.
However, we refrain from this optimization in this work as we obtain sufficiently tight results without loosening \cref{eq:vertex-no-overlap} for our use cases.

%% file: sections/evaluation.tex
\section{Experimental Results} \label{sec:evaluation}

We have implemented our approach%
\footnote{Code available at: \url{https://github.com/toladnertum/paper-pose-repeatability}}
using the MATLAB toolbox CORA~\cite{Althoff2015ARCH,Althoff2025manual}, which offers a wide range of set-based computing functionalities, and conducted all experiments on an 11th Gen Intel(R) Core(TM) i7-11800H processor with 64 GB of memory.
Evaluation details and additional experiments are given in \cref{sec:exp-details-and-ablation-studies}. 

\subsection{Experiment on Synthetic Data: Landing Plane}
We simulate a landing plane with the camera facing the runway.
The camera has a resolution of $200\times 200$ and a focal length $f=250$ (\cref{eq:intrinsic-params}).
The pose space $\poseSpace$ is chosen large enough to capture the entire landing maneuver:
\begin{align}
    \begin{split}
        \poseSpace &= [-50,50] \times  [-50,150] \times [50,350] \iftwocol \\\TWOCOLAMP\else \fi \times [0,90] \times [-5,5] \times [-5,5] \subset \R^6,
    \end{split}
\end{align}
where the values are chosen to match a typical landing maneuver,\footnote{Angles are converted to radians in the implementation.} and are restricted to poses where the target is visible in the image.
Examples, along with the resulting certified pose estimates, are visualized in \cref{fig:exp-synthetic}:
We notice that our approach can correctly and tightly identify the pose, particularly when the camera is closer to the target.
Please note that for larger distances, many images indeed become very similar visually (\cref{fig:exp-synthetic}a-b),
and thus the exact pose becomes more ambiguous due to the surjection of the image generation (\cref{def:camera}).

\begin{figure*}
    \centering
    \includegraphics[alt={Full state space with identified certified pose estimates per image (a)-(e). Close targets result in tighter certified pose estimates visualized as smaller sets.}]{./figures/externalize\iftwocol -twocol\else \fi/figures/evaluation/plane_landing/pose_eval_plane_landing_estimates}\\
    \includegraphics[alt={Progressively taken images as the plane gets closer to the runway}]{./figures/externalize\iftwocol -twocol\else \fi/figures/evaluation/plane_landing/pose_eval_plane_landing_images}
    \caption{Landing plane scenario: Images and resulting certified pose estimates $\poseUncertain^*$.}
    \label{fig:exp-synthetic}
\end{figure*}

In our experiment, we randomly sample $100$ poses from $\poseSpace$ and generate images using different targets on a runway (\cref{fig:system-model}).
Then, we recover the poses using our approach.
The results are shown in \cref{tab:exp-synthetic}.
For each sample, the true pose is contained in the obtained certified pose estimate, validating the soundness of our approach.
We note that the initial filtering of the pose candidates (\cref{alg:image-to-pose}, line~\ref{alg-line:quick-exit}) can rule out many pose candidates,
which is only possible due to a tight computation of the offline part (\cref{sec:pose-to-image}),
as a larger linearization error results in a larger area enclosing the target.
The subsequently imposed constraints again substantially reduce the volume of $\poseUncertain^*$ while barely increasing the computation time.

\begin{table*}[]
    \centering
    \caption{
        Comparison of computation time and resulting volume of the certified pose estimate $\poseUncertain^*$ using different targets:
        Initial filtering using the offline pre-computations (baseline, \texttt{Filter}), and our full approach (\texttt{Ours}). \vspace{0.5\baselineskip}
    }
    \begin{tabular}{c C C R@{$\pm$}L R@{$\pm$}L R@{$\pm$}L R@{$\pm$}L R@{$\pm$}L}
        \toprule
        & \multicolumn{1}{c}{} & \multicolumn{3}{c}{\textbf{\#Candidates}} & \multicolumn{4}{c}{\textbf{Online Time [s]}} & \multicolumn{4}{c}{\textbf{Norm. Vol. [\%]}}   \\
        \cmidrule(lr{.75em}){3-5}  \cmidrule(lr{.75em}){6-9}   \cmidrule(lr{.75em}){10-13}
        \textbf{Target} & \multicolumn{1}{c}{\textbf{Offline [h]}} & \numPoseCandidates & \multicolumn{2}{c}{\texttt{Filter}} & \multicolumn{2}{c}{\texttt{Filter}} & \multicolumn{2}{c}{\texttt{Ours}} & \multicolumn{2}{c}{\texttt{Filter}} & \multicolumn{2}{c}{\texttt{Ours}}   \\
        \midrule
        \texttt{30}      & 1.86 & 5979 & 41.6 & 9.3 & 0.23 & 0.08 & 2.81 & 1.02 & 1.29 & 0.33 & \mathbf{0.16} & \mathbf{0.41} \\
        \texttt{R}       & 0.53 & 3966 & 28.5 & 6.9 & 0.12 & 0.04 & 2.10 & 0.82 & 4.28 & 1.24 & \mathbf{0.30} & \mathbf{0.97} \\
        \texttt{Stripes} & 0.17 & 1405 & 30.8 & 7.2 & 0.05 & 0.02 & 2.10 & 0.78 & 4.82 & 1.24 & \mathbf{0.64} & \mathbf{1.71} \\
        \bottomrule
    \end{tabular}
    \label{tab:exp-synthetic}
\end{table*}

\subsection{Experiment on Real-World Data: Sign Detection}
We additionally demonstrate our approach on a real-world dataset consisting of $138$ images recorded by a SilkyEvCam event-based camera with calibrated true poses $\pose^*$ using a Vicon motion capture system.
The camera has an image resolution of $640 \times 480$ and focal length $f=533.33$.
In this experiment, we consider the task of detecting a slow-moving vehicle sign in front of the camera,
with a large enough pose space $\poseSpace\subset \R^6$ for the considered setting:
\begin{align}
    \begin{split}
        \poseSpace &= [-0.5,0.5] \times  [-0.5,0.5] \times [0.5,1.5] \iftwocol \\\TWOCOLAMP\else \fi \times [-5,5] \times [-5,5] \times [-15,15].
    \end{split}
\end{align}

\begin{figure*}
    \centering
    \begin{minipage}{0.38\linewidth}
        \centering
        \includegraphics[alt={(a) Raw traffic sign image with a lot of noise. (b) Cleaned image has significantly less noise.}]{./figures/externalize\iftwocol -twocol\else \fi/figures/evaluation/real/pose_real_images}
        \captionof{figure}{Example image of recorded dataset: (a) raw image (\texttt{Raw}), and (b) with basic denoising applied (\texttt{Cleaned}).}
        \label{fig:exp-real}
    \end{minipage} \hspace{0.4cm}
    \begin{minipage}{0.5\linewidth}
        \centering
        \vspace{-0.5\baselineskip}
        \captionof{table}{Results using our approach (\texttt{Ours}) on both methods shown in \cref{fig:exp-real}.}
        \label{tab:exp-real}
        \ \\
        \begin{tabular}{l R@{$\pm$}L R@{$\pm$}L R@{$\pm$}L}
            \toprule
            \textbf{Method} & \multicolumn{2}{c}{\textbf{\#Cand.}} & \multicolumn{2}{c}{\textbf{Time [s]}} & \multicolumn{2}{c}{\textbf{Vol. [\%]}} \\
            \midrule
            \texttt{Raw}     & 88.7 & 20.2 & 3.39 & 0.80 & 5.54 & 1.84 \\
            \texttt{Cleaned} & 48.6 & 15.2 & 2.14 & 0.69 & 1.17 & 0.50 \\
            \bottomrule
        \end{tabular}
    \end{minipage}
\end{figure*}

We show in \cref{fig:exp-real}a an example image from this dataset.
Please note that the event-based camera reliably detects the edges of the sign,
with additional noise spread across the entire image.
We empirically determined the maximum noise level $\imageNoiseLevel = (w\cdot h)\cdot 1\% = 3072$ pixels,
and compare our approach on the raw images from the camera and on images with basic denoising applied (\cref{fig:exp-real}b).
We initially partition the pose space $\poseSpace$ into $\numPoseCandidates=666$ candidates and pre-compute the respective properties offline.
Subsequently, our approach quickly and accurately localizes online the true pose $\pose^*$ on all images as shown in \cref{tab:exp-real}.
We again note that the true pose was contained in the certified pose estimate for all images in this experiment.
Please note that the cleaned images substantially reduce the online computation time as fewer candidates have to be considered.
Thus, any progress in formal filtering methods, as explored in \cite{santa2023certified}, directly benefits our approach.

\subsection{Standalone Experiment on Sound Image Enclosures}

We want to stress that our sound image enclosure (\cref{sec:pose-to-image}) also has applications beyond pose estimation, such as in abstract rendering~\cite{jiabstract}.
Thus, we provide additional results solely on the image enclosure in this experiment.
We first show some qualitative examples in \cref{fig:pose_quali} depicting the tightness of our approach:
From an uncertain pose $\poseUncertain$ and a target, we compute the resulting transformations into the pixel coordinate frame (PCF) to obtain the enclosure $\V_{i(\cdot,k)}^\text{PCF}$ (\cref{prop:algForward}).
Subsequently, we generate $5$ (extreme) random samples from $\poseUncertain$ and transform them individually (\cref{def:camera}).
Thus, the vertices of the transformed samples in PCF must be contained in the computed enclosures.
The results demonstrate the tightness of our approach in (a) rotational uncertainty, (b) zoom uncertainty (z-axis), and (c) translational uncertainty (x- and y-axis) for different targets,
as the vertices of the samples are spread within the entire respective set $\V_{i(\cdot,k)}^\text{PCF}$.
Please note that areas in the set that do not contain any samples do not necessarily reflect outer approximation.
As a measure of this outer approximation, the average ratio between the interval hull radius of the linearized term and the linearization error in \cref{eq:hypercube-vertex} is $18.6\%$ in the experiments presented above;
details on this measure are provided in \cref{sec:implementation-details}.

\begin{figure*}
    \centering
    \includegraphics[alt={Visualization of effect of different perturbations on the final projection}]{./figures/externalize\iftwocol -twocol\else \fi/figures/evaluation/quali/pose_quali}
    \caption{
        Qualitative examples on the sound image enclosure:
        (a) Angle perturbation on target \texttt{30}.
        (b) Zoom uncertainty ($z$-axis) on target \texttt{R}.
        (c) Translational uncertainty ($x$- and $y$-axis) on target \texttt{Stripes}.
        Target $\target^\text{PCF}$ is shown for the center of the perturbed pose $\poseUncertain$, respectively.
    }
    \label{fig:pose_quali}
\end{figure*}

\iftwocol \subsection{Ablation Studies} \else \subsubsection{Ablation studies.} \fi
Additional experiments are provided in \cref{sec:ablation-studies}.

%% file: sections/related-work.tex
\section{Related Work} \label{sec:related-work}

While several techniques exist to localize an object~\cite{placed2023survey,ebadi2023present,cadena2017past,haralick1989pose},
including approaches based on deep learning~\cite{chen2023deep},
obtaining formal guarantees of the pose solely from an image is largely unexplored:
Only the local robustness of deep-learning-based object detectors~\cite{chowdhury2025robustness} and pose estimators~\cite{luo2025certifying,jiabstract} can be verified.
While formal neural network verification has progressed rapidly in recent years~\cite{kaulen2025vnn},
its focus is limited to verifying local robustness properties~\cite{johnson2025neural}.

The geometric transformations of the camera~(\cref{def:camera}) can be modeled as a neural network~\cite{santa2022nnlander},
such that, when concatenated with a controller, robustness queries can be utilized to verify the control actions of local pose spaces.
Image-invariant regions in the pose space can also be utilized to verify vision-based controllers~\cite{habeeb2023verification}.
We refer interested readers to the survey in \cite{mitra2024formal} for further approaches applying formal verification to vision-based deep learning.
However, the black-box nature of neural networks prevents giving global guarantees of their correctness~\cite{goodfellow2015explaining}.
Modeling the camera as a synthetic neural network also enables formal noise reduction of noisy images during the computation of certified pose estimates~\cite{santa2023certified,cruz2026certified};
however, these approaches do not consider the full six-dimensional pose space $\poseSpace$ we are able to verify in our approach.

Certified pose estimation can also be obtained from partial point clouds~\cite{yang2020teaser,talak2023certifiable},
including the relative change of position between two frames~\cite{garcia2021certifiable}.
Conformal prediction~\cite{angelopoulos2023conformal,yang2023object} can also provide guarantees under the considered dataset.
It is also worth noting that global pose localization can be obtained by using QR codes as landmarks~\cite{zhang2015localization,nazemzadeh2017indoor}, but the precise location of the agent relative to the QR code remains unknown.
Our approach could also be used in set-based observers \cite{chen2015disturbance,alanwar2023distributed,raissi2011interval} to update the current set of states given an image.

%% file: sections/conclusion.tex
\section{Conclusion and Next Steps} \label{sec:conclusion}

We show that certified pose estimates can be obtained solely from an image and a known target geometry.
Our experiments demonstrate the effectiveness of our approach,
thereby reliably localizing an agent relative to the target in just over a second.
This potentially opens the door to real-time capabilities if implemented and optimized in a more efficient programming language than MATLAB.
In addition, GPU acceleration can be applied as all operations are based on simple matrix manipulations in a batch-wise manner.

We acknowledge several limitations of this work, each of which opens up interesting research directions that we leave for future investigation.
Firstly, our approach makes assumptions about the current setting, e.g., that the target is clearly visible in the image.
However, this might not occur in practice, as weather conditions can affect the visual perception, and malicious actors can manipulate the target.
Secondly, noisy images make finding a small subset of witness pixels to refine the offline pre-computed pose candidates more challenging.
While we show that our approach still efficiently localizes the agent even for unprocessed images,
we also show that denoising helps to find a tighter certified pose estimate,
and further research on both the algorithm for finding the witness pixels and the refinement through \cref{prop:enclosing-unsafe-inputs} would benefit our approach.
Thirdly, we only consider a single (target) object present in the image, which might be unrealistic in complex settings.
Please note that if the images of the respective objects are separated and the distances between multiple objects are known,
one can compute an intersection of the obtained cones (\cref{fig:exp-synthetic}) to obtain an even tighter result due to the different angles of the object with respect to the camera.
Despite these limitations, we believe that our work is a significant step toward guaranteeing the safety of autonomous agents in safety-critical settings.

%% file: sections/acknowledgements.tex
\iftwocol \textbf{\ackname.} \else \subsubsection{\ackname} \fi
This research was partially supported by the German Research Foundation (Deutsche Forschungsgemeinschaft, DFG) under grant number AL 1185/33-1, the German Academic Exchange Service (Deutscher Akademischer Austauschdienst, DAAD) under grant number 57751853, NSF awards CNS
\#2504809 and \#2313104 and the UCI ProperAI Institute, an Engineering+Society Institute funded as part of a generous gift from Susan and Henry Samueli.

\iftwocol \textbf{\discintname.} \else \subsubsection{\discintname} \fi
The authors have no competing interests to declare that
are relevant to the content of this article.

\iftwocol \textbf{Data-Availability Statement.} \else \subsubsection{Data-Availability Statement} \fi
All code and data used for the experiments are made publicly available through the referenced GitHub repository and under \url{https://doi.org/10.5281/zenodo.19640859}.
%

%% file: sections/appendix.tex

\iftwocol \else {\noindent \Large \textbf{Appendix}} \fi

\iftwocol
  \let\section\subsection
  \let\subsection\subsubsection
  \let\subsubsection\paragraph
\fi

\section{On Set-Based Computing}

\subsection{Polynomial Zonotopes and Image Enclosures}
\label{sec:pZ}

We construct a running example to illustrate the image enclosure (\cref{prop:image-enclosure}) using polynomial zonotopes (\cref{def:pZ-mat}) in this sub-section.
For simplicity, we only use regular polynomial zonotopes~\cite{kochdumper2020sparse}; the construction with their matrix equivalent is analogous.
Consider an uncertain angle $\theta\in\contSet{O}=[\pi/6,\, \pi/2]\subset\R$.
As a polynomial zonotope, this set is represented by:
\begin{equation}
    \label{eq:example-angles}
    \contSet{O} = \shortPZ{[\sfrac{\pi}{3}]}{[\sfrac{\pi}{6}]}{[\ ]}{[1]}.
\end{equation}
Thus, the offset of the resulting polynomial zonotope is at the center of $\contSet{O}$ with a single generator spanning the radius to cover the entire range.
For multiple angles $[\theta_x\ \theta_y\ \theta_z]\in\contSet{O}'\subset\R^3$, the construction is analogous:
\newcommand{\opCenter}[2][]{\operatortt[#1]{center}{#2}}
\newcommand{\opRadius}[2][]{\operatortt[#1]{radius}{#2}}
\begin{equation}
    \contSet{O}' = \shortPZ{\opCenter{\contSet{O}'}}{\diag{\opRadius{\contSet{O}'}}}{[\ ]}{I_3}.
\end{equation}
The generator matrix is given by a diagonal matrix, as the angles can vary independently of each other.
In our approach, we are required to enclose the output of the $\sin(x)$ and $\cos(x)$ functions given an uncertain input (\cref{alg:pose-to-image}, line~\ref{alg-line:sin-cos}).
We illustrate this process in \cref{fig:pose_pZ_sin_cos}:
To obtain a tight enclosure, we first determine the domain $\contSet{D}$ by computing bounds of the input angles $\contSet{O}$.
Within this domain, we compute an approximation $p(x)$ of $f(x)$ using regression and determine the maximal approximation error $d$ within this domain.
Represented as a polynomial zonotope, these two steps can be computed using \cref{eq:PZ-mat-mul} and \cref{eq:PZ-mink-sum}, respectively.

\begin{figure*}
    \centering
    \includegraphics[alt={Enclosure is given by two parallel lines enclosing the actual function output}]{./figures/externalize\iftwocol -twocol\else \fi/figures/pz_sin_cos/pose_pZ_sin_cos}
    \caption{Example image enclosure of (a) $\sin(x)$ and (b) $\cos(x)$ on a given domain $\contSet{D}$.}
    \label{fig:pose_pZ_sin_cos}
\end{figure*}

For the angles in \cref{eq:example-angles},
the respective approximations and approximation errors are:
\begin{equation}
    \begin{aligned}
        p_{\sin}(x) &= 0.4851 x + 0.2981, & d_{\sin} &= 0.0602, \text{ and} \\
        p_{\cos}(x) &= -0.8402 x + 1.3432, & d_{\cos} &= 0.0374,
    \end{aligned}
\end{equation}
which is also illustrated in \cref{fig:pose_pZ_sin_cos}.
Thus, written as a polynomial zonotope, the enclosure is:
\begin{align}
    \begin{split}
        &\cmatrix{\sin{(\contSet{O})} \\ \cos{(\contSet{O})}} \subseteq \iftwocol \\\TWOCOLAMP \else \fi \shortPZ[.]{\cmatrix{0.8061 \\ 0.4634}}{\cmatrix{0.2540 \\ -0.4399}}{\cmatrix{0.0602 & 0 \\ 0 & 0.0374}}{[1]} \hspace{-0.1\linewidth}
    \end{split}
\end{align}
Please note that the dependent generator matrix (second entry, \cref{def:pZ-mat}) is not a diagonal matrix,
as it is computed from the same uncertain angles $\contSet{O}$ and thus encodes dependence (i.e., it is projected from the same hypercube).
Propagating such dependencies through \cref{alg:pose-to-image} enables us to exploit them in \cref{alg:image-to-pose} by imposing constraints on the shared hypercube.

\subsection{Uncertain Target Geometry Through Reference Points}
\label{sec:uncertain-target-geometry}

Please note that the transformation from TCF to CCF is linear in the given vertices $V_i^\text{TCF}$ (\cref{eq:camera-analog}).
For this reason, we can compute the transformation by computing \cref{eq:camera-analog}, independently of the number of polygons and their vertices, using only three non-collinear reference points $\refPoint^\text{TCF}\in\R^{3\times 3}$,
as any point $p\in\R^3$ can be reconstructed from $\refPoint^\text{TCF}$ as follows:
\begin{equation}
    \label{eq:reference-points-parametrization}
    p = \lambda_{(1)} \refPoint_{(:,1)}^\text{TCF} + \lambda_{(2)} \refPoint_{(:,2)}^\text{TCF} + \lambda_{(3)} \refPoint_{(:,3)}^\text{TCF},
\end{equation}
where $\lambda\in\R^3$ is found by solving the following linear system of equations:
\begin{equation}
    \refPoint^\text{TCF}\lambda = p.
\end{equation}
Given uncertainty, this reformulation and reconstruction can directly be computed using \cref{eq:PZ-mat-mul} and \cref{eq:PZ-mink-sum}.
Moreover, it naturally enables us to introduce uncertainty in the shape of the target $\target$ by adding uncertainty to $\lambda$.
Let $\refPoint=I_3$, any uncertainty added to $\lambda$ directly results in uncertainty in the three axis directions.
For planar targets (\cref{fig:motivation-planar-target}), i.e., where all polygons are oriented on the hyperplane with $c=\cmatrix{0&0&1}^\transpose$, $d=0$ (\cref{def:polygon}),
we can choose $\refPoint=[[I_2\ \mathbf{0}]^\transpose\ \mathbf{0}]\in\R^{3\times 3}$,
which effectively eliminates $\lambda_{(3)}$ as $\refPoint_{(\cdot,3)}=\mathbf{0}$.

\subsection{Convex Hulls of Uncertain Vertices}
\label{sec:convHull}

To obtain a tight convex hull over all $\V_{i(\cdot,k)}^\text{PCF}$ to compute pixel containment in the entire polygon,
we deploy an algorithm based on support functions summarized in \cref{alg:supportFunc}.

\begin{algorithm}[b!]
    \caption{Convex Hull Using Support Functions}
    \label{alg:supportFunc}
    \begin{algorithmic}[1]
        \Require {Uncertain vertices $\V_{i}^\text{PCF}$}
        \State Gather directions $C_i\in\R^{\numSFs\times 2}$, init. $d\gets\mathbf{0}\in\R^\numSFs$ \label{alg-line:supportFunc-start} \Comment{Support function enclosure}
        \For{$j\in[\numSFs]$}
            \Comment{(\cref{fig:supportFunc})}
            \State $d_{i(j)} \gets \max_{k\in[\numPGvertices_i]} \opSupportFunc[inline]{\V_{i(\cdot,k)}^\text{PCF}}{C_{i(j,\cdot)}}$  \Comment{\cref{eq:supportFunc}}
        \EndFor
        \State \Return $\widehat{\PG}_i^\text{PCF} \gets \shortPH{C_i}{d_i} \subset\R^2$ \label{alg-line:supportFunc-end} \Comment{$\widehat{\PG}_i^\text{PCF} \supseteq \opConvHull{\V_i^\text{PCF}}$}
    \end{algorithmic}
\end{algorithm}

Let us first recall the definition of a support function~\cite[Def.~1]{girard2008efficient}:
Given a set $\contSet{S}\subset\R^n$ and a direction $a\in\R^n$,
\begin{align}
    \label{eq:supportFunc}
    \begin{split}
        \opSupportFunc{\contSet{S}}{a} = \min &\quad b\in\R, \\
        \text{s.t.} &\quad\forall s\in\contSet{S}\colon a^\transpose s \leq b.
    \end{split}
\end{align}
Thus, $\contSet{H}=\shortPH{a^\transpose}{b}$ forms a halfspace constraint such that $\contSet{S}\subseteq \contSet{H}$.
For multiple constraints, i.e., $A\in\R^{\numSFs\times n}$ with $b\in\R^\numSFs$ computed for each constraint individually,
$\contSet{P} = \shortPH{A}{b}$ forms a polytope such that $\contSet{S}\subseteq\contSet{P}$.
If $\contSet{S}$ is represented as a polynomial zonotope, the inequality in \cref{eq:supportFunc} is often not strict~\cite[Prop.~7]{kochdumper2020sparse}.

\begin{figure*}
    \centering
    \includegraphics[alt={As in Fig. 4, but with explicit construction of the convex hull in (b) through the intersection of halfspaces}]{./figures/externalize\iftwocol -twocol\else \fi/figures/supportFunc/pose-supportFunc}
    \caption{
        Example of the support function enclosure for a target with a single polygon ($\numPGs=1$) with uncertain vertices $\V_1^\text{PCF}$ via \cref{alg:pose-to-image},
        for a pose $\poseUncertain$ with perturbed angles $\angleX,\angleY,\angleZ$ by $10$ degrees.
    }
    \label{fig:supportFunc}
\end{figure*}

In our case, for each polygon $i\in[\numPGs]$,
\begin{equation}
    \contSet{S}_i=\bigcup_{k\in[\numPGvertices_i]} \V_{i(\cdot,k)}^\text{PCF}
\end{equation}
to obtain an enclosure of the respective polygon $\widehat{\PG}_i^\text{PCF}\subset\R^2$.
By construction, any $A_i\in\R^{\numSFs\times 2}$ yields a sound enclosure;
however, we deploy the following heuristics to obtain tight results:
(i) To preserve the overall shape, we choose directions orthogonal to the edges of the polygon,
which are computed via the offset of two subsequent uncertain vertices $\V_{i(\cdot,k)}^\text{PCF}$, $\V_{i(\cdot,k+1)}^\text{PCF}$ (\cref{def:pZ-mat}).
(ii) To sharpen the corners, we add the direction from the center of the polygon (computed by the mean of the vertex offsets) to each vertex offset.
An example with these directions is shown in \cref{fig:supportFunc}b.
Please note that if two subsequent uncertain vertices are influenced by $\poseUncertain$ to a different degree,
the direction orthogonal to the edge (heuristic (i)) can be very conservative.
For example, compare the bottom two vertices with the top two vertices in \cref{fig:supportFunc}b,
and notice the gap between the top right vertex and the halfspace boundary to its right.
This can be iteratively refined by selecting the points $v_{i,k}\in\V_{i(\cdot,k)}^\text{PCF}$, $v_{i,k+1}\in\V_{i(\cdot,k+1)}^\text{PCF}$ for which the support function is realized (\cref{eq:supportFunc}), respectively,
and then (iii) additionally computing the support function in the direction orthogonal to the line connecting $v_{i,k}$ and $v_{i,k+1}$.

\paragraph{On convex hulls of polynomial zonotopes.}
Interested readers might wonder why the convex hull is computed via support function enclosures and not -- as all other operations of our approach -- with polynomial zonotopes using~\cite[Prop.~14]{kochdumper2020sparse}.
We identify two main reasons against this approach:
(i) The computation of the convex hull as polynomial zonotopes comes with additional complexity in the set representation, such that pixel inclusion checks become difficult to compute.
(ii) Witness pixels are computed per vertex, so there is no additional benefit of carrying the dependencies into the convex hull such that they can be exploited via the shared hypercube.

\iftwocol \begin{figure} \else \begin{wrapfigure}{r}{0.28\textwidth} \fi
    \centering
    \iftwocol\else\vspace{-2\baselineskip}\fi
    \includegraphics[alt={Convex hull of two polynomial zonotopes}]{./figures/externalize\iftwocol -twocol\else \fi/figures/supportFunc/pZ_convHull}
    \iftwocol\else\vspace{-\baselineskip}\fi
    \caption{Convex hull of two polynomial zonotopes.}
    \label{fig:pZ_convHull}
    \iftwocol\else\vspace{-3\baselineskip}\fi
\iftwocol \end{figure} \else \end{wrapfigure} \fi

To demonstrate the former, let us consider two simple polynomial zonotopes:
\begin{align}
    \begin{split}
        \PZ_1 &= \shortPZ[,]{\cmatrix{-1 \\ -1}}{\sfrac{1}{2}  \cmatrix{1\ & 0 \\ 0 & 1}}{\cmatrix{\ \\\ }}{\cmatrix{1\ & 0 \\ 0 & 1}} \\
        \PZ_2 &= \shortPZ[,]{\cmatrix{1 \\ 1}}{\sfrac{1}{2}\cmatrix{1\ & 1 \\ 1 & -1}}{\cmatrix{\ \\\ }}{\cmatrix{1\ & 0 \\ 0 & 1}}
    \end{split}
\end{align}
where $\PZ_1$ represents a box centered at $[-1\ -1]^\transpose$ and $\PZ_2$ represents a diamond shape centered at $[1\ 1]^\transpose$ (\cref{fig:pZ_convHull}).
Unfortunately, $\widehat{\PZ}=\opConvHull{\PZ_1,\PZ_2}$ using~\cite[Prop.~14]{kochdumper2020sparse} results in a large outer approximation,
and $\widehat{\PZ}$ has $33$ nonlinear monomials, which makes the pixel inclusion check discussed earlier difficult to compute.

\section{Proofs}\label{sec:proofs}

We provide all missing proofs in this section in the order they appear in the main body of the paper. We also restate all statements for convenience.


\renewcommand{\thetheorem}{\ref{prop:algForward}}
\propalgForward*
\begin{proof}
    The implication follows as each applied operation to compute $\image$ is outer-approximative.
    In particular, to compute $\widehat{\contSet{P}}_i^\text{PCF}$, the operations defined in \cref{prop:image-enclosure}, \cref{eq:PZ-mat-mul} and \cref{eq:PZ-mink-sum} are used,
    and the support function enclosure (\cref{eq:supportFunc}) is computed for polynomial zonotopes using \cite[Prop.~7]{kochdumper2020sparse}.

    Please note that the number of generators to represent $\contSet{V}_{i(\cdot,k)}^\text{PCF}$ as a polynomial zonotope is constant, as we initially always have $6$ generators representing the uncertainty in $\poseSpace\subset\R^6$, and the number of operations to obtain each $\contSet{V}_{i(\cdot,k)}^\text{PCF}$ is constant, each only adding a fixed number of generators:
    Given a polynomial zonotope $\PZ_1\subset\R^n$, \cref{prop:image-enclosure} adds at most $n$ generators to $\PZ$. Let $\PZ_1$ have $h_1$ generators and a second $\PZ_2$ have $h_2$ generators with proper dimensions, applying \cref{eq:PZ-mat-mul} results in $\mathcal{O}(h_1h_2)$ generators, and applying \cref{eq:PZ-mink-sum} results in $h_1{+}h_2$ generators.
    Thus, in \cref{alg:pose-to-image}, all computations leading up to the convex hull (line \ref{alg-line:convHull}) are negligible in the computational complexity analysis.
    As we compute this convex hull using support functions in $\numSFs=3\numPGvertices_i$ directions (\cref{sec:convHull}),
    and as the number of polygon vertices $\numPGvertices_i \ll w\cdot h$, the overall complexity is dominated by the pixel inclusion check in line \ref{alg-line:pixel-inclusion} and is thus given by $\mathcal{O}(\numPGs wh)$.
    Please note that both the inner \textbf{for} loop in \cref{alg:pose-to-image} and \cref{alg:supportFunc} can be computed efficiently in a batch-wise manner.
     \iftwocol \qedhere \else \hfill \qed \fi
\end{proof}


\renewcommand{\thetheorem}{\ref{prop:enclosing-unsafe-inputs}}
\propenclosingunsafeinputs*
\begin{proof}
    While the original statement in \cite[Prop. 2]{koller2025out} is provided for zonotopes~\cite[Def.~2]{kochdumper2020sparse},
    we have polynomial zonotopes as inputs.
    However, the approach is still applicable as 
    the input set $\contSet{X}$ is effectively a zonotope as it has the identity as exponent matrix~\cite[Def.~2]{kochdumper2020sparse},
    we take the linearized (zonotopic) part $\widetilde{G}_y$ to compute $C$,
    and the computation of $d$ is outer-approximative via a zonotope enclosure~\cite[Prop.~5]{kochdumper2020sparse}.
     \iftwocol \qedhere \else \hfill \qed \fi
\end{proof}


\renewcommand{\thetheorem}{\ref{th:algBackward}}
\thalgBackward*
\begin{proof}
    The soundness of our approach follows from the outer-approximative computation in \cref{alg:pose-to-image} (\cref{prop:algForward}) and the application of \cref{prop:enclosing-unsafe-inputs}, which ensures that if $\pose^*\in\poseUncertain_c$, the constraints $C_c\leq d_c$ only restrict $\poseUncertain_c$ such that $\pose^*\in\cZ{\poseUncertain_c}{C_c}{d_c}$.

    Given the pre-computation of all candidate poses using \cref{alg:pose-to-image}, the remaining work for each candidate, polygon, and vertex, is finding the witness pixels, which can be done in $\mathcal{O}(wh)$ using the pre-computed bitmap, and applying \cref{prop:enclosing-unsafe-inputs} given by $\mathcal{O}(\numSFs wh)$.
    As $\numSFs$ is a small constant (\cref{sec:convHull}),
    the overall computational complexity is $\mathcal{O}(\numPoseCandidates\numPGs \numPGvertices_{i} wh)$.
     \iftwocol \qedhere \else \hfill \qed \fi
\end{proof}


\section{Evaluation Details and Ablation Studies}
\label{sec:exp-details-and-ablation-studies}

\subsection{Evaluation Details}
\label{sec:implementation-details}

We assume that a pixel of an image is turned on if the polygon $\target^\text{PCF}$ intersects with the respective area (\cref{fig:implementation-details}a).
This assumption is reasonable for event-based cameras.
Please note that this intersection has to be computed whenever pixel containment is checked in the presented algorithms.

\begin{figure*}
    \centering
    \includegraphics[alt={(a) Image shows that any intersection of a pixel with the target turns on the pixel. (b) Hierarchical partitioning results in smaller sub-cubes near the camera}]{./figures/externalize\iftwocol -twocol\else \fi/figures/evaluation/details/pose_implementation_details}
    \caption{
        (a) We assume pixels are turned on if the (unobservable) target $\target$ intersects with the respective area of the pixel.
        (b) Pose candidates $\poseUncertain_c$ partitioning pose space $\poseSpace$ in a hierarchical manner.
    }
    \label{fig:implementation-details}
\end{figure*}

We determine the pose space by analyzing the maximum poses at which anything meaningful is visible in the resulting image.
Afterward, the pose candidates are determined by partitioning the pose space $\poseSpace$ in a hierarchical manner (\cref{fig:implementation-details}b):
Given a pose candidate $\poseUncertain$,
we pass $\poseUncertain$ through \cref{alg:pose-to-image},
and analyze the ratio between the radius of the interval hull between the linear approximating generators and the error generators (\cref{eq:hypercube-vertex}).
If this ratio exceeds a certain threshold $\errorRate\in\R_{+}$,
the pose candidate $\poseUncertain$ is split in half in each dimension
-- where we use sensitivity analysis to determine the dimensions to split --
and the process is repeated.
Otherwise, a sufficiently tight pose candidate $\poseUncertain$ is found and added to the list.
The process starts with the entire pose space $\poseUncertain=\poseSpace$.
Usually, finding the pose candidates takes substantially less time than computing the subsequent properties of each pose candidate (e.g., pixel containments),
and both are shown combined as \emph{offline} computation time in this work.
Please note that the target cannot be seen from some poses, and we remove these as well while partitioning $\poseSpace$;
such filtered candidates are thus not visible in \cref{fig:implementation-details}b.
Moreover, we show the geometry of the targets used in our experiments in \cref{fig:pose-targets}.

\begin{figure*}
    \centering
    \includegraphics[alt={Coordinates of various targets in TCF}]{./figures/externalize\iftwocol -twocol\else \fi/figures/evaluation/targets/pose_targets}
    \caption{
        Geometry of used targets:
        (a) \texttt{30},
        (b) \texttt{R},
        (c) \texttt{Stripes}, and
        (d) \texttt{SlowVehicle}.
    }
    \label{fig:pose-targets}
\end{figure*}

\subsection{Ablation Studies}
\label{sec:ablation-studies}
In this sub-section, we provide further insights into the influence of hyperparameters on our approach using ablation studies.

\subsubsection{Varying the image resolution.}

\begin{table*}[]
    \centering
    \caption{
        Ablation study on the image resolution.
    }
    \begin{tabular}{C  C R@{$\pm$}L R@{$\pm$}L R@{$\pm$}L}
        \toprule
        \textbf{Image Resolution} & \textbf{\#Candidates }\numPoseCandidates & \multicolumn{2}{c}{\texttt{Filter [s]}} & \multicolumn{2}{c}{\textbf{Time [s]}}  & \multicolumn{2}{c}{\textbf{\#Witness Pixels}} \\
        \midrule
        100 \times 100 & 286 &    24.5 & 3.5   &  1.47 & 0.46  & \ \ \ \ \ \  3.8 & 1.6  \\
        200 \times 200 & 1405 &  29.5 & 5.7 &    2.17 & 0.50     &  4.2 & 1.4 \\
        300 \times 300 & 1857 &   28.8 & 7.0 &   2.84 & 0.85 &   3.3 & 1.4 \\
        500 \times 500 & 2592 &   30.1 & 6.1 &  4.50 & 1.15 &   2.8 & 1.1  \\
        \bottomrule
    \end{tabular}
    \label{tab:ablation-image-resolution}
\end{table*}

\begin{figure*}
    \centering
    \includegraphics[alt={Continuation of Fig. 4b and 12b, where the figure shows that using progressively more halfspaces results in a tighter enclosure}]{./figures/externalize\iftwocol -twocol\else \fi/figures/evaluation/ablation_sf/pose_ablation_sf}
    \caption{Comparison of the number of support function evaluations $\numSFs$ (showing zoomed-in view of setting in \cref{fig:supportFunc}): (a) only in directions of the edges, (b) + directions to corners, (c-d) + iterative refinement.}
    \label{fig:supportFunc-heuristic-ablation}
\end{figure*}

We summarize in \cref{tab:ablation-image-resolution} the effect of the image resolution on the results.
Interestingly, despite requiring more pose candidates with larger image resolution, after the initial filtering is applied, the average number of remaining candidates (\cref{alg:image-to-pose}, line~\ref{alg-line:quick-exit}) stays roughly constant, and so does the computation time.
We also show the average number of witness pixels that have to be considered with each image resolution, confirming the efficiency of our approach.

Please note that the computational complexities in \cref{prop:algForward} and \cref{th:algBackward} can be improved by cropping a high-resolution image $\image\in\B^{w\times h}$ to the relevant part, and adjusting the application of \cref{prop:enclosing-unsafe-inputs} due to this rescaling (not implemented).
However, we did not implement this optimization to avoid distorting the overall evaluation results.

\subsubsection{Varying the support function heuristics.}
In \cref{alg:pose-to-image}, we discuss a convex hull enclosure by running $\numSFs$ support function evaluations.
The detailed process is described in \cref{sec:convHull}, and is also illustrated in \cref{fig:supportFunc}.

In this ablation study, we provide further insights by presenting qualitative examples of different heuristics for selecting support function directions in \cref{fig:supportFunc-heuristic-ablation}.
In particular, we show the result for the direction orthogonal to the edge connecting the center of the respective vertices (\cref{fig:supportFunc-heuristic-ablation}a), the directions to the vertices themselves (\cref{fig:supportFunc-heuristic-ablation}b), and one additional iterative heuristic (\cref{fig:supportFunc-heuristic-ablation}c-d):
When evaluating the direction orthogonal to an edge, we obtain a point from each vertex where this support is realized (\cref{eq:supportFunc}).
Connecting these two points, and again running a support function evaluation orthogonal to this line, usually provides a better outer approximation of the true shape.
We have found that choosing $\numSFs>3\cdot\numPGvertices$ only has diminishing returns.

%% file: bib.bib
@inproceedings{cruz2026certified,
    author = {Santa Cruz, Ulices and Elfar, Mahmoud and Shoukry, Yasser},
    title = {Correct-by-Construction: {Vision}-based Pose Estimation using Geometric Generative Models},
    booktitle = {arxiv},
    doi = {10.48550/arXiv.2601.17556},
    year = {2026}
}

@inproceedings{kochdumper2020sparse,
    author = {Kochdumper, Niklas and Althoff, Matthias},
    booktitle = {IEEE Transactions on Automatic Control},
    year = {2020},
    title = {Sparse polynomial zonotopes: {A} novel set representation for reachability analysis},
    volume = {66},
    issue = {9},
    pages = {4043-4058},
    doi = {10.1109/TAC.2020.3024348},
}

@inproceedings{kochdumper2022open,
    title={Open- and closed-loop neural network verification using polynomial zonotopes},
    author={Kochdumper, Niklas and Schilling, Christian and Althoff, Matthias and Bak, Stanley},
    booktitle={{NASA} Formal Methods Symposium},
    doi = {10.1007/978-3-031-33170-1\_2},
    volume = {13903},
    pages = {16-36},
    year={2023}
}

@inproceedings{katz2017reluplex,
    author = {Katz, Guy and Barrett, Clark and Dill, David L. and Julian, Kyle and Kochenderfer, Mykel J.},
    booktitle = {International Conference on Computer Aided Verification},
    year = {2017},
    doi = {10.1007/978-3-319-63387-9\_5},
    volume = {10426},
    pages = {97–117},
    title = {Reluplex: An efficient {SMT} solver for verifying deep neural networks}
}

@inproceedings{kaulen2025vnn,
    title={The 6th International Verification of Neural Networks Competition ({VNN-COMP} 2025): {Summary} and Results},
    author={Kaulen, Konstantin and Ladner, Tobias and Bak, Stanley and Brix, Christopher and Duong, Hai and Flinkow, Thomas and Johnson, Taylor T. and Koller, Lukas and Manino, Edoardo and Nguyen, ThanhVu H. and Hoaze, Wu},
    booktitle={arXiv preprint arXiv:2512.19007},
    doi = {10.48550/arXiv.2512.19007},
    year={2025}
}

@inproceedings{ladner2025formal,
    title={Formal Verification of Graph Convolutional Networks with Uncertain Node Features and Uncertain Graph Structure},
    author={Ladner, Tobias and Eichelbeck, Michael and Althoff, Matthias},
    booktitle={Transactions on Machine Learning Research},
    year = {2025},
}

@inproceedings{ladner2023automatic,
    title={Automatic abstraction refinement in neural network verification using sensitivity analysis},
    author={Ladner, Tobias and Althoff, Matthias},
    booktitle={Proceedings of the 26th {ACM} International Conference on Hybrid Systems: Computation and Control},
    doi = {10.1145/3575870.3587129},
    volume = {26},
    pages = {1-13},
    year={2023}
}

@book{de2008computational,
    title={Computational geometry: {Algorithms} and applications},
    author={De Berg, Mark and Cheong, Otfried and Van Kreveld, Marc and Overmars, Mark},
    year={2008},
    publisher={Springer}
}

@inproceedings{girard2008efficient,
    title={Efficient reachability analysis for linear systems using support functions},
    author={Girard, Antoine and Le Guernic, Colas},
    booktitle={IFAC Proceedings Volumes},
    volume = {41},
    issue = {2},
    pages = {8966-8971},
    doi = {10.3182/20080706-5-KR-1001.01514},
    year={2008}
}

@inproceedings{koller2025out,
    title={Out of the Shadows: {Exploring} a Latent Space for Neural Network Verification},
    author={Koller, Lukas and Ladner, Tobias and Althoff, Matthias},
    booktitle={International Conference on Learning Representations},
    year={2026}
}

@inproceedings{scott2016constrained,
    title={Constrained zonotopes: {A} new tool for set-based estimation and fault detection},
    author={Scott, Joseph K. and Raimondo, Davide M. and Marseglia, Giuseppe Roberto and Braatz, Richard D.},
    booktitle={Automatica},
    doi = {10.1016/j.automatica.2016.02.036},
    volume = {69},
    pages = {126-136},
    year={2016}
}

@inproceedings{grigorescu2020survey,
  title={A survey of deep learning techniques for autonomous driving},
  author={Grigorescu, Sorin and Trasnea, Bogdan and Cocias, Tiberiu and Macesanu, Gigel},
  booktitle={Journal of Field Robotics},
  doi = {10.1002/rob.21918},
  volume = {37},
  issue = {3},
  pages = {362-386},
  year={2020}
}

@phdthesis{althoff2010reachability,
  title={Reachability analysis and its application to the safety assessment of autonomous cars},
  author={Althoff, Matthias},
  year={2010},
  school={Technische Universit{\"a}t M{\"u}nchen}
}

@inproceedings{goodfellow2015explaining,
	title	 = {Explaining and harnessing adversarial examples},
	author	 = {Ian Goodfellow and Jonathon Shlens and Christian Szegedy},
	year	 = {2015},
	booktitle	 = {International Conference on Learning Representations}
}

@inproceedings{chen2023deep,
  title={Deep learning for visual localization and mapping: {A} survey},
  author={Chen, Changhao and Wang, Bing and Lu, Chris Xiaoxuan and Trigoni, Niki and Markham, Andrew},
  booktitle={IEEE Transactions on Neural Networks and Learning Systems},
  doi = {10.1109/TNNLS.2023.3309809},
  volume = {35},
  issue = {12},
  pages = {17000-17020},
  year={2023}
}

@inproceedings{pirayesh2022jamming,
  title={Jamming attacks and anti-jamming strategies in wireless networks: {A} comprehensive survey},
  author={Pirayesh, Hossein and Zeng, Huacheng},
  booktitle={IEEE Communications Surveys \& Tutorials},
  doi = {10.1109/COMST.2022.3159185},
  volume = {24},
  issue = {2},
  pages = {767-809},
  year={2022}
}

@inproceedings{wu2024gps,
  title={{GPS} jamming: {A} historical record from global radio occultation ({RO}) observations},
  author={Wu, Dong L. and Csar, Cornelius and Salinas, Jude H.},
  booktitle={International Technical Meeting of the Satellite Division of The Institute of Navigation},
  doi = {10.33012/2024.19816},
  volume = {37},
  pages = {781-795},
  year={2024}
}

@inproceedings{placed2023survey,
  title={A survey on active simultaneous localization and mapping: {State} of the art and new frontiers},
  author={Placed, Julio A and Strader, Jared and Carrillo, Henry and Atanasov, Nikolay and Indelman, Vadim and Carlone, Luca and Castellanos, Jos{\'e} A},
  booktitle={IEEE Transactions on Robotics},
  doi = {10.1109/TRO.2023.3248510},
  volume = {24},
  issue = {2},
  pages = {767-809},
  year={2023}
}

@book{icao2018Aerodrome,
    title = {Annex 14, Volume I, Aerodrome Design and Operations},
    author = {{International Civil Aviation Organization}},
    publisher = {ICAO},
    year = {2018}
}

@misc{runwaywebsite,
    title = {Runway Stripes And Markings, Explained.},
    author = {Boldmethod},
    year = {2025},
    url = {https://www.boldmethod.com/learn-to-fly/regulations/runway-markings-and-spacing/}
}

@inproceedings{Althoff2015ARCH,
	author			= {Matthias Althoff},
	title			= {An Introduction to {CORA} 2015},
	year			= {2015},
	booktitle		= {Proc. of the 1st and 2nd Workshop on Applied Verification for Continuous and Hybrid Systems},
	pages			= {120-151},
    doi		    	= {10.29007/zbkv}
  }

@inproceedings{garcia2015comprehensive,
    title={A comprehensive survey on safe reinforcement learning},
    author={Garc{\i}a, Javier and Fern{\'a}ndez, Fernando},
    booktitle={Journal of Machine Learning Research},
    volume = {16},
    pages = {437-1480},
    year={2015}
}

@inproceedings{clarke2000counterexample,
       title={Counterexample-guided abstraction refinement},
       author={Clarke, Edmund and Grumberg, Orna and Jha, Somesh and Lu, Yuan and Veith, Helmut},
       booktitle={International Conference on Computer Aided Verification},
       doi = {10.1007/10722167\_15},
       volume = {1855},
       pages = {154–169},
       year={2000}
}

@book{hartley2003multiple,
    title={Multiple view geometry in computer vision},
    author={Hartley, Richard and Zisserman, Andrew},
    year={2003},
    volume = {2},
    publisher={Cambridge University Press}
}

@book{goldstein1980classical,
    title={Classical mechanics},
    author={Goldstein, Herbert},
    publisher={Addison Wesley},
    edition = {2},
    year={1980}
}

@inproceedings{ebadi2023present,
  title={Present and future of slam in extreme environments: {The} {DARPA} {SubT} Challenge},
  author={Ebadi, Kamak and Bernreiter, Lukas and Biggie, Harel and Catt, Gavin and Chang, Yun and Chatterjee, Arghya and Denniston, Christopher E and Desch{\^e}nes, Simon-Pierre and Harlow, Kyle and Khattak, Shehryar and others},
  booktitle={IEEE Transactions on Robotics},
  doi = {10.1109/TRO.2023.3323938},
  volume = {40},
  pages = {936-959},
  year={2023}
}

@inproceedings{cadena2017past,
  title={Past, present, and future of simultaneous localization and mapping: {Toward} the robust-perception age},
  author={Cadena, Cesar and Carlone, Luca and Carrillo, Henry and Latif, Yasir and Scaramuzza, Davide and Neira, Jos{\'e} and Reid, Ian and Leonard, John J},
  booktitle={IEEE Transactions on Robotics},
  doi = {10.1109/TRO.2016.2624754},
  volume = {32},
  issue = {6},
  pages = {1309-1332},
  year={2017}
}

@inproceedings{haralick1989pose,
  title={Pose estimation from corresponding point data},
  author={Haralick, Robert M. and Joo, Hyonam and Lee, Chung-Nan and Zhuang, Xinhua and Vaidya, Vinay G and Kim, Man Bae},
  booktitle={IEEE Transactions on Systems, Man, and Cybernetics},
  doi = {10.1109/21.44063},
  volume = {19},
  issue = {6},
  pages = {1426-1446},
  year={1989}
}

@inproceedings{zhang2015localization,
  title={Localization and navigation using {QR} code for mobile robot in indoor environment},
  author={Zhang, Huijuan and Zhang, Chengning and Yang, Wei and Chen, Chin-Yin},
  booktitle={IEEE International Conference on Robotics and Biomimetics},
  doi = {10.1109/ROBIO.2015.7419715},
  pages = {2501-2506},
  year={2015}
}

@inproceedings{nazemzadeh2017indoor,
  title={Indoor localization of mobile robots through {QR} code detection and dead reckoning data fusion},
  author={Nazemzadeh, Payam and Fontanelli, Daniele and Macii, David and Palopoli, Luigi},
  booktitle={IEEE/ASME Transactions on Mechatronics},
  doi = {10.1109/TMECH.2017.2762598},
  volume = {22},
  issue = {6},
  pages = {2588-2599},
  year={2017}
}

@inproceedings{johnson2025neural,
  title={Is Neural Network Verification Useful and What Is Next?},
  author={Johnson, Taylor T.},
  booktitle={Allerton Conference on Communication, Control, and Computing Proceedings},
  volume = {61},
  year={2025}
}

@inproceedings{chowdhury2025robustness,
  title={Robustness Verification for Object Detectors Using Set-Based Reachability Analysis},
  author={Chowdhury, Sayak and Khandelwal, Hardik and D’Souza, Meenakshi},
  booktitle={International Conference on Artificial Neural Networks},
  pages={481--492},
  doi = {10.1007/978-3-032-04546-1\_39},
  year={2025}
}

@inproceedings{luo2025certifying,
  title={Certifying Robustness of Learning-Based Keypoint Detection and Pose Estimation Methods},
  author={Luo, Xusheng and Wei, Tianhao and Liu, Simin and Wang, Ziwei and Mattei-Mendez, Luis and Loper, Taylor and Neighbor, Joshua and Hutchison, Casidhe and Liu, Changliu},
  booktitle={ACM Transactions on Cyber-Physical Systems},
  doi = {10.1145/3728362},
  volume = {9},
  issue = {2},
  pages = {1-26},
  year={2025}
}

@inproceedings{jiabstract,
  title={Abstract Rendering: {Certified} Rendering Under {3D} Semantic Uncertainty},
  author={Ji, Chenxi and Li, Yangge and Zhong, Xiangru and Zhang, Huan and Mitra, Sayan},
  booktitle={Advances in Neural Information Processing Systems},
  volume = {38},
  pages = {105171-105198},
year={2025}
}

@inproceedings{gallego2020event,
  title={Event-based vision: {A} survey},
  author={Gallego, Guillermo and Delbr{\"u}ck, Tobi and Orchard, Garrick and Bartolozzi, Chiara and Taba, Brian and Censi, Andrea and Leutenegger, Stefan and Davison, Andrew J and Conradt, J{\"o}rg and Daniilidis, Kostas and others},
  booktitle={IEEE Transactions on Pattern Analysis and Machine Intelligence},
  doi = {10.1109/TPAMI.2020.3008413},
  volume = {44},
  issue = {1},
  pages = {154-180},
  year={2020}
}

@inproceedings{talak2023certifiable,
    title={Certifiable object pose estimation: {Foundations}, learning models, and self-training},
    author={Talak, Rajat and Peng, Lisa R. and Carlone, Luca},
    booktitle={IEEE Transactions on Robotics},
    doi = {10.1109/TRO.2023.3271568},
    volume = {39},
    issue = {4},
    pages = {2805-2824},
    year={2023}
}

@inproceedings{yang2020teaser,
    title={{TEASER}: {Fast} and certifiable point cloud registration},
    author={Yang, Heng and Shi, Jingnan and Carlone, Luca},
    booktitle={IEEE Transactions on Robotics},
    doi = {10.1109/TRO.2020.3033695},
    volume = {37},
    issue = {2},
    pages = {314-333},
    year={2020}
}

@inproceedings{garcia2021certifiable,
    title={Certifiable relative pose estimation},
    author={Garcia-Salguero, Mercedes and Briales, Jesus and Gonzalez-Jimenez, Javier},
    booktitle={Image and Vision Computing},
    doi = {10.1016/j.imavis.2021.104142},
    volume={109},
    issue = {104142},
    year={2021}
}

@inproceedings{angelopoulos2023conformal,
    title={Conformal prediction: {A} gentle introduction},
    author={Angelopoulos, Anastasios N. and Bates, Stephen},
    booktitle={Foundations and Trends in Machine Learning},
    doi = {10.1561/2200000101},
    volume = {16},
    issue = {4},
    pages = {494–591},
    year={2023}
}

@inproceedings{yang2023object,
    title={Object pose estimation with statistical guarantees: {Conformal} keypoint detection and geometric uncertainty propagation},
    author={Yang, Heng and Pavone, Marco},
    booktitle={Proceedings of the IEEE/CVF conference on computer vision and pattern recognition},
    pages = {8947-8958},
    year={2023}
}

@inproceedings{chen2015disturbance,
    title={Disturbance-observer-based control and related methods -- {An} overview},
    author={Chen, Wen-Hua and Yang, Jun and Guo, Lei and Li, Shihua},
    booktitle={IEEE Transactions on Industrial Electronics},
    doi = {10.1109/TIE.2015.2478397},
    year={2015}
}

@inproceedings{alanwar2023distributed,
    title={Distributed set-based observers using diffusion strategies},
    author={Alanwar, Amr and Rath, Jagat Jyoti and Said, Hazem and Johansson, Karl Henrik and Althoff, Matthias},
    booktitle={Journal of the Franklin Institute},
    doi = {10.1016/j.jfranklin.2023.03.025},
    volume = {360},
    issue = {10},
    pages = {6976--6993},
    year={2023}
}

@inproceedings{raissi2011interval,
    title={Interval state estimation for a class of nonlinear systems},
    author={Ra{\"\i}ssi, Tarek and Efimov, Denis and Zolghadri, Ali},
    booktitle={IEEE Transactions on Automatic Control},
    doi = {10.1109/TAC.2011.2164820},
    volume = {57},
    issue = {1},
    pages = {260-265},
    year={2011}
}

@inproceedings{bronnimann2006space,
    title={Space-efficient algorithms for computing the convex hull of a simple polygonal line in linear time},
    author={Br{\"o}nnimann, Herv{\'e} and Chan, Timothy M.},
    booktitle={Computational Geometry},
    doi = {10.1016/j.comgeo.2005.11.005},
    volume = {34},
    issue = {2},
    pages = {75-82},
    year={2006}
}

@inproceedings{santa2022nnlander,
    title={{NNLander-VeriF}: {A} neural network formal verification framework for vision-based autonomous aircraft landing},
    author={Santa Cruz, Ulices and Shoukry, Yasser},
    booktitle={NASA Formal Methods Symposium},
    doi = {10.1007/978-3-031-06773-0\_11},
    volume = {13260},
    pages = {213–230},
    year={2022}
}

@inproceedings{santa2023certified,
    title={Certified vision-based state estimation for autonomous landing systems using reachability analysis},
    author={Santa Cruz, Ulices and Shoukry, Yasser},
    booktitle={IEEE Conference on Decision and Control},
    doi = {10.1109/CDC49753.2023.10384107},
    volume = {62},
    year={2023}
}

@inproceedings{mitra2024formal,
    title={Formal Verification Techniques for Vision-Based Autonomous Systems -- {A} Survey},
    author={Mitra, Sayan and P{\u{a}}s{\u{a}}reanu, Corina and Prabhakar, Pavithra and Seshia, Sanjit A and Mangal, Ravi and Li, Yangge and Watson, Christopher and Gopinath, Divya and Yu, Huafeng},
    booktitle={Principles of Verification: Cycling the Probabilistic Landscape},
    volume = {15262},
    pages={89--108},
    doi = {10.1007/978-3-031-75778-5\_5},
    year={2024}
}

@inproceedings{habeeb2023verification,
  title={Verification of camera-based autonomous systems},
  author={Habeeb, P. and Deka, Nabarun and D’Souza, Deepak and Lodaya, Kamal and Prabhakar, Pavithra},
  booktitle={IEEE Transactions on Computer-Aided Design of Integrated Circuits and Systems},
  doi = {10.1109/TCAD.2023.3240131},
  volume = {42},
  issue = {10},
  pages = {3450-3463},
  year={2023}
}

@article{Althoff2025manual,
    title			= {{CORA} Manual},
    author			= {Althoff, Matthias and Kochdumper, Niklas and Ladner, Tobias and Perschl, Maximilian and Wetzlinger, Mark},
    journal 		= {Technical University of Munich},
    year			= {2025},
    url			= {https://cora.in.tum.de/manual}
}
